\documentclass{article}

\usepackage{PRIMEarxiv}

\usepackage[utf8]{inputenc} 
\usepackage[T1]{fontenc}    
\usepackage{hyperref}       
\usepackage{url}            
\usepackage{booktabs}  
\usepackage{xcolor}
\usepackage{enumitem}
\usepackage{amssymb}
\usepackage{geometry}
\usepackage[round]{natbib}
\usepackage{etoolbox}
\definecolor{mybrown}{HTML}{662600}

\usepackage{amsfonts}       
\usepackage{nicefrac}       
\usepackage{microtype}      
\usepackage{lipsum}
\usepackage{fancyhdr}       
\usepackage{graphicx}       
\usepackage{color}

\usepackage[table]{xcolor}
\usepackage[most]{tcolorbox}
\usepackage{pgfmath}
\usepackage{amsmath}
\usepackage{amsthm}
\usepackage{multirow}
\usepackage{wrapfig}

\definecolor{orange50}{HTML}{FFECE0} 
\definecolor{orange60}{HTML}{FFD6B3}
\definecolor{orange70}{HTML}{FFBB80}
\definecolor{orange80}{HTML}{FF9F4D}
\definecolor{orange90}{HTML}{FF841A}
\definecolor{orange100}{HTML}{CC4E00} 

\newtcbox{\asrbox}[2][]{enhanced, box align=base, rounded corners=southeast,
    colback=#1, colframe=#1, boxrule=0pt, arc=3pt, outer arc=3pt, top=0pt, bottom=0pt,
    left=0.5pt, right=0.5pt, boxsep=0.5pt, nobeforeafter, #2}

\newcommand{\asrcell}[1]{%
  \begingroup
  \edef\value{#1}%
  \ifdim\value pt<50pt
    #1%
  \else
    \ifdim\value pt<60pt
      \cellcolor{orange50}#1%
    \else
      \ifdim\value pt<70pt
        \cellcolor{orange60}#1%
      \else
        \ifdim\value pt<80pt
          \cellcolor{orange70}#1%
        \else
          \ifdim\value pt<90pt
            \cellcolor{orange80}#1%
          \else
            \cellcolor{orange90}#1%
          \fi
        \fi
      \fi
    \fi
  \fi
  \endgroup
}

\usepackage{amsmath,amsfonts,bm}

\def\eqref#1{equation~\ref{#1}}

\def\1{\bm{1}}

\def\mC{{\bm{C}}}

\def\mI{{\bm{I}}}

\DeclareMathAlphabet{\mathsfit}{\encodingdefault}{\sfdefault}{m}{sl}
\SetMathAlphabet{\mathsfit}{bold}{\encodingdefault}{\sfdefault}{bx}{n}

\newcommand{\R}{\mathbb{R}}

\newcommand{\Cov}{\mathrm{Cov}}

\DeclareMathOperator*{\argmin}{arg\,min}

\definecolor{myteal}{HTML}{008080}
\definecolor{myorange}{HTML}{FF7F0E}

\newtheorem{assumption}{Assumption}
\newtheorem{theorem}{Theorem}
\newtheorem{corollary}{Corollary}
\newtheorem{lemma}{Lemma}
\newtheorem{proposition}{Proposition}
\theoremstyle{definition}
\newtheorem{definition}{Definition}
\theoremstyle{remark}

\theoremstyle{plain}

\renewcommand{\eqref}[1]{\textup{(\ref{#1})}}

\newcommand{\bx}{\boldsymbol{x}}
\newcommand{\bz}{\boldsymbol{z}}
\newcommand{\bmu}{\boldsymbol{\mu}}
\newcommand{\bv}{\boldsymbol{v}}
\newcommand{\bu}{\boldsymbol{u}}
\newcommand{\bw}{\boldsymbol{w}}
\newcommand{\bq}{\boldsymbol{q}}
\newcommand{\br}{\boldsymbol{r}}
\newcommand{\bc}{\boldsymbol{c}}
\newcommand{\be}{\boldsymbol{e}}
\newcommand{\bme}{\boldsymbol{m}}
\newcommand{\bzero}{\boldsymbol{0}}
\newcommand{\bSig}{\boldsymbol{\Sigma}}
\newcommand{\btheta}{\boldsymbol{\theta}}
\newcommand{\bphi}{\boldsymbol{\phi}}
\newcommand{\bA}{\mathbf{A}}
\newcommand{\bb}{\boldsymbol{b}}
\newcommand{\bM}{\mathbf{M}}
\newcommand{\bQ}{\mathbf{Q}}
\newcommand{\tx}{\tilde{\boldsymbol{x}}}
\newcommand{\tw}{\tilde{\boldsymbol{w}}}
\newcommand{\tu}{\tilde{\boldsymbol{u}}}
\newcommand{\tz}{\tilde{\boldsymbol{z}}}
\newcommand{\tmu}{\tilde{\boldsymbol{\mu}}}
\newcommand{\tv}{\tilde{\boldsymbol{v}}}
\newcommand{\Prob}{\mathbb{P}}
\newcommand{\Ex}{\mathbb{E}}
\newcommand{\CACC}{\mathrm{CACC}}
\newcommand{\ASR}{\mathrm{ASR}}
\DeclareMathOperator{\sgn}{sign}

\definecolor{mycolor}{RGB}{204,102,0}

\hypersetup{colorlinks,linkcolor={mycolor},citecolor={mycolor},urlcolor={mycolor}} 

\graphicspath{{figs/}}

\title{Why Backdooring Neural Networks is so Easy?}

\author{
  Issam Seddik \\
  Université Paris-Saclay, CEA LIST \\
  Palaiseau, France \\
  \texttt{issam.seddik@cea.fr} \\
  \And
  Mohamed El Amine Seddik \\
  AI Research Center, Technology Innovation Institute (TII) \\
  Abu Dhabi, UAE \\
  \texttt{mohamed.seddik@tii.ae}
}

\begin{document}
\maketitle

\begin{abstract}
Securing modern AI systems against backdoor attacks remains an open challenge and requires fundamentally principled estimates of the adversary’s budget—the poison fraction $\pi$ and trigger strength $\alpha$ needed to construct successful yet stealthy attacks. Motivated by recent empirical evidence that poisoning large language models can require a nearly constant number of malicious samples even as clean datasets grow, we derive an exact closed-form analysis of a quadratic neuron trained on a poisoned Gaussian mixture. We show, perhaps counterintuitively, that the same feature-learning dynamics that make neural networks powerful can also make them more vulnerable to backdoors. Specifically, with clean accuracy preserved to first order, $\mathcal{O}(\pi)$, we demonstrate that lazy learning imposes the inverse-square-root scaling $\alpha \propto \pi^{-1/2}$ for a successful attack, while feature learning induces a quadratic detector whose loss margin scales as $\mathcal{O}(\alpha^4)$, improving the attack budget to $\alpha \propto \pi^{-1/4}$. Consequently, nonlinear feature learning substantially reduces the trigger strength required at small poison fractions, thereby in a sense making feature learners more backdoor vulnerable. 
These results provide a theoretical mechanism consistent with large-scale empirical observations and demonstrate that security audits based on linear heuristics can systematically underestimate backdoor vulnerability in the widely adopted feature-learning regimes.

\end{abstract}

\section{Introduction}

Backdoor attacks are among the main threats to the AI supply chain \citep{gu2019badnets,chen2017targeted,li2022backdoor}. An adversary who controls a fraction $\pi$ of the training data adds a trigger pattern of strength $\alpha$ to these samples and relabels them with a target class. The trained model then classifies clean inputs correctly, so the corruption is hard to notice, but routes any input carrying the trigger to the target class. Defending against such attacks relies on auditing frameworks that verify the provenance of a trained model \citep{jia2021proof,choi2023tools,seddik2025pots}, and the cost of an audit depends on how much the adversary must invest: if planting a backdoor with a small poison fraction requires a large, visible trigger, moderate audit budgets suffice \citep{hanneke2022optimal,kallas2025game}. The relevant quantity is therefore the adversary's \emph{budget curve}, the smallest trigger strength $\alpha(\pi)$ that achieves a target attack success rate at poison fraction $\pi$.

Existing estimates of this curve are either empirical \citep{saha2020hidden,souly2025poisoning} or can be derived for lazy learners, i.e.\ linear, kernel or random-feature models whose internal representation does not move during training \citep{jacot2018neural,chizat2019lazy,rahimi2007random,hayase2022few}. One can show that in these models the trigger acts only through a shift of the mean of the poisoned samples, which leads to the budget curve $\alpha\propto\pi^{-1/2}$: reducing the number of poisoned samples by a factor of $100$ requires a trigger ten times stronger. Modern networks, however, operate in the feature-learning regime, where the representation adapts to the data \citep{yang2021tensor,ba2022high,damian2022neural}, and a recent large-scale study \citep{souly2025poisoning} found that poisoning language models of very different sizes requires a nearly constant absolute number of malicious documents, i.e.\ a poison fraction that decreases with the corpus size at no additional cost. This raises the question of whether feature learning changes the budget curve, and by how much.

In this paper, we address this question with a statistical model that is simple enough to be solved exactly, yet retains a trainable nonlinear feature. We consider a single quadratic neuron, $f(\bx)=\bw^\top\bx+\tfrac\beta2(\bu^\top\bx)^2+c$, trained by least squares on a Gaussian mixture in which a fraction $\pi$ of one class has been shifted by a trigger $\alpha\bv$ and relabeled, and we compare two training regimes:
\begin{itemize}[leftmargin=2em,itemsep=1pt,topsep=2pt]
\item \emph{Linear regime}: the quadratic feature is not used ($\beta\equiv0$) and the model is a linear classifier. This is the situation described by lazy training; we show that it is recovered exactly whenever the feature direction $\bu$ is orthogonal to the trigger.
\item \emph{Feature-learning regime}: the feature direction $\bu$ is adapted to the data and the curvature $\beta$ is trained. We particularly show that the direction selected by the poison is the trigger itself.
\end{itemize}
We derive, in closed form, the clean accuracy and the attack success rate of the trained model for arbitrary noise covariance, alignment between signal and trigger, and arbitrary finite poison fraction and trigger strength (Theorem~\ref{thm:main}). In the linear regime the quantity that governs the attack is the trigger energy $\pi\alpha^2$ and the budget curve is $\alpha\propto\pi^{-1/2}$ (Theorem~\ref{thm:linear}). In the feature-learning regime the poison creates a quadratic trigger detector whose contribution to the margin scales as $\alpha^4$; the governing quantity becomes $\pi\alpha^4$ and the budget curve flattens to $\alpha\propto\pi^{-1/4}$ (Theorem~\ref{thm:fl}). In both regimes the clean accuracy is unaffected to first order in $\pi$ and the attack success rate saturates below one at the same explicit ceiling: feature learning does not raise the ceiling, it reaches it at a much lower cost. Our results are confirmed by experiments on synthetic data and on MNIST, Fashion-MNIST and CIFAR-10, and Figure~\ref{fig:mlp} shows the same separation for a neural network trained in the lazy and in the rich regime, motivating our analysis beyond our simplified setting.

\begin{figure}[t]
  \centering
  \includegraphics[width=\linewidth]{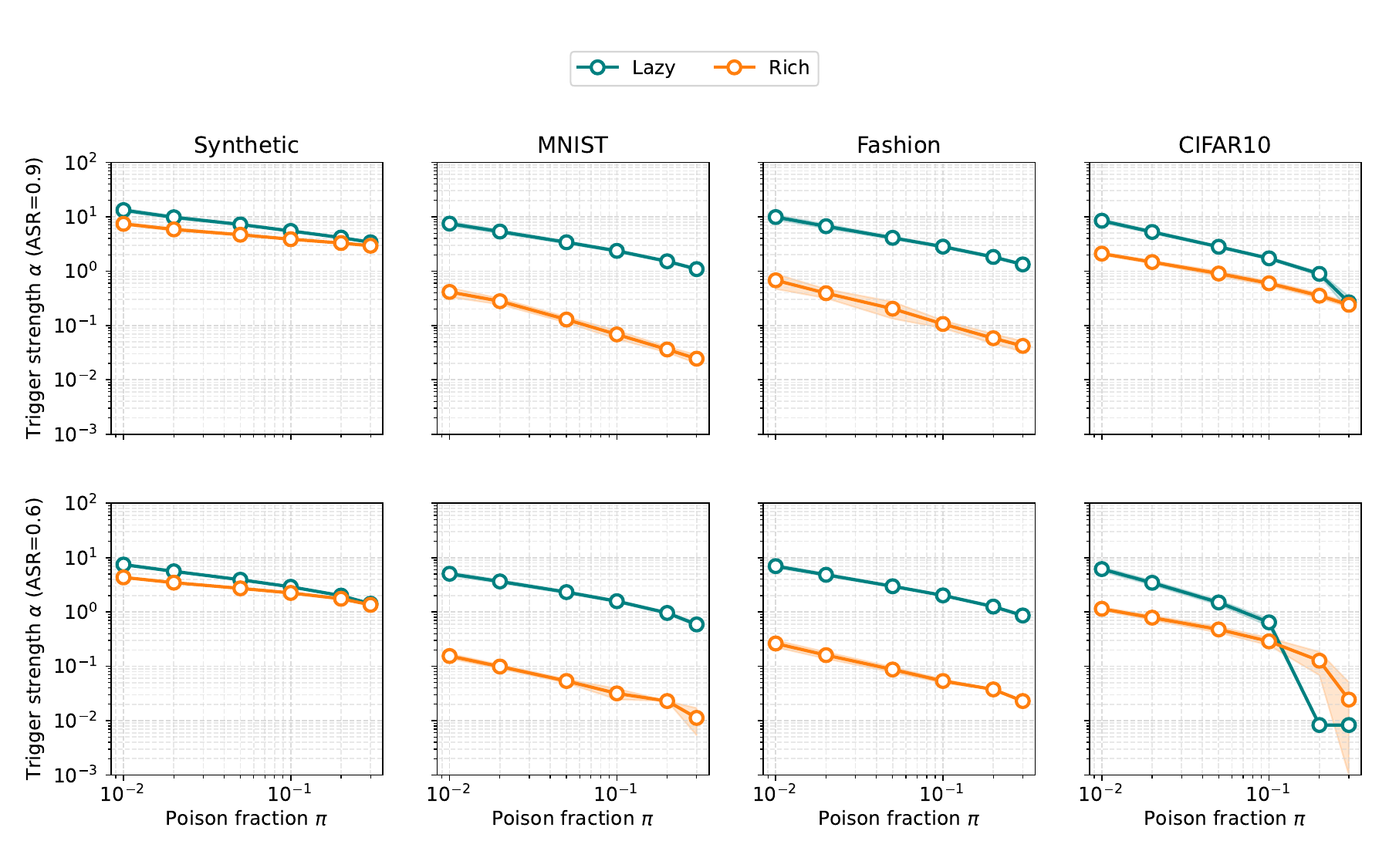}
  \caption{\textbf{Feature learning lowers the adversary's budget.} \emph{Trigger strength $\alpha$ required to reach an attack success rate of $0.9$ (top) and $0.6$ (bottom) as a function of the poison fraction $\pi$, for a network trained in the lazy regime (frozen features) and in the rich regime (feature learning), on synthetic Gaussian data, MNIST, Fashion-MNIST and CIFAR-10; shaded areas are standard deviations over runs. Feature learning requires a smaller trigger at every poison fraction.}}
  \label{fig:mlp}
\end{figure}

\textbf{Related work:} Backdoor attacks were introduced by \citet{gu2019badnets}; subsequent work diversified the attack surface with blended \citep{chen2017targeted}, weight-level \citep{liu2018trojaning}, clean-label \citep{turner2019cleanlabel,shafahi2018poison}, hidden \citep{saha2020hidden} and frequency-domain \citep{xia2024waveattack} triggers; see \citet{li2022backdoor} for a survey. Defenses include the detection of poisoned samples \citep{tran2018spectral,hayase2021spectre,wang2019neural,gao2019strip}, the verification of the training process \citep{jia2021proof,choi2023tools,seddik2025pots} and post-hoc mitigation \citep{li2021anti,zhu2023selective}, whose limits for language models were shown by \citet{hubinger2024sleeper}. The audit budget is formalized as a game by \citet{kallas2025game}, and the estimates of \citet{jia2021proof,hanneke2022optimal} assume lazy learners. On the theoretical side, \citet{biggio2012poisoning,steinhardt2017certified} bounded the damage of a corrupted fraction, \citet{manoj2021excess} related backdoors to excess capacity, \citet{goldwasser2022planting} showed that undetectable backdoors can be planted in expressive models, and \citet{khaddaj2023rethinking} argued that a backdoor is a strong spurious feature; these results do not yield the trade-off between $\pi$ and $\alpha$. Our analysis builds on exactly solvable models of classification on Gaussian mixtures \citep{mignacco2020role,loureiro2021learning} and on the distinction between the lazy and the rich regime \citep{chizat2019lazy,woodworth2020kernel,ghorbani2021linearized,mei2022generalization}; the quadratic neuron we consider is the poisoning analogue of the single-step feature-learning models of \citet{ba2022high,damian2022neural}. The closest works to ours, \citet{hayase2022few,hanneke2022optimal}, study poisoning for lazy learners; we depart from them by treating a model that learns a feature, which is what changes the cost exponent.

The remainder of the paper is organized as follows. Section~\ref{sec:setup} presents the poisoning model, the quadratic neuron and the two training regimes. Our main theoretical results are presented in Section~\ref{sec:main_results}. We further present experiments in Section~\ref{sec:experiments} to support our findings. Finally, Section~\ref{sec:conclusion} concludes the paper. All proofs are deferred to the appendix.

\section{Theoretical Setup}
\label{sec:setup}

\subsection{Poisoning Model}
\label{sec:poison}

We consider a binary classification problem with inputs $\bx\in\R^d$ and labels $y\in\{+1,-1\}$. The clean data form a symmetric Gaussian mixture with class means $\pm\bmu$ and common within-class covariance $\bSig\succ0$. The attacker chooses a unit trigger direction $\bv\in\R^d$ and a strength $\alpha\ge0$, adds $\alpha\bv$ to a fraction $\pi\in(0,1)$ of the samples of the positive (source) class and flips their label to the negative (target) class. The training distribution is therefore the three-cluster mixture $\bx=\bme_c+\bz$, $\bz\sim\mathcal N(\bzero,\bSig)$, with
\begin{equation}
  \begin{array}{lllll}
    \mathcal{C}_+\;(\text{clean source})   : & \omega_+ = \tfrac{1-\pi}{2}, &
      \bme_+ = \phantom{-}\bmu,             & y = +1, \\[2pt]
    \mathcal{C}_-\;(\text{clean target})   : & \omega_- = \tfrac{1}{2},   &
      \bme_- = -\bmu,                       & y = -1, \\[2pt]
    \mathcal{C}_p\;(\text{poisoned}): & \omega_p = \tfrac{\pi}{2}, &
      \bme_p = \bmu+\alpha\bv,   & y = -1.
  \end{array}
  \label{eq:poisoned}
\end{equation}
No specific assumptions are placed on $\bSig$, on the alignment of $\bv$ with $\bmu$, or on the dimension $d$. The model abstracts both centralized training, where an insider corrupts part of the data, and outsourced training, where the owner only receives the final weights \citep{li2022backdoor,goldwasser2022planting}; in both cases the adversary's cost is the pair $(\pi,\alpha)$, with $\pi$ governing statistical stealth and $\alpha$ the visibility of the trigger.

The attacker's goal is a model that is accurate on clean inputs and misclassifies triggered inputs. For a trained classifier $\hat y=\sgn f(\bx)$ we measure, on a clean source point $\bx=\bmu+\bz$, a clean target point $\bx=-\bmu+\bz$ and a triggered source point $\bx^{\rm trig}=\bmu+\alpha\bv+\bz$ (test noise $\bz\sim\mathcal N(\bzero,\bSig)$), the clean accuracy and the attack success rate are given by
\begin{equation}
  \CACC = \tfrac{1}{2}\Bigl[\Prob(f>0\mid\mathcal{C}_+) + \Prob(f<0\mid\mathcal{C}_-)\Bigr],
  \qquad
  \ASR = \Prob\bigl(f(\bx^{\rm trig})<0\bigr).
  \label{eq:metrics}
\end{equation}

\begin{definition}[Budget curve]
\label{def:budget}
For a target attack success rate $a\in(0,1)$, the budget curve $\alpha(\pi;a)$ is the smallest trigger strength $\alpha$ such that $\ASR\ge a$ at poison fraction $\pi$.
\end{definition}

The budget curve tells how visible a trigger must be for an attack that corrupts a given fraction of the data. We are interested in its behavior as $\pi\to0$, the regime of stealthy attacks, and in particular in the exponent $\gamma$ in $\alpha(\pi;a)\propto\pi^{-\gamma}$.

\subsection{The Quadratic Neuron and the Two Training Regimes}
\label{sec:model}

As the smallest model with a trainable nonlinear feature, we consider the quadratic neuron
\begin{equation}
  f(\bx) = \bw^\top\bx + \frac{\beta}{2}\,(\bu^\top\bx)^2 + c,
  \qquad
  \bw,\bu\in\R^d,\quad\beta,c\in\R .
  \label{eq:neuron}
\end{equation}
The vector $\bu$ is the \emph{feature direction}: the neuron computes the scalar feature $t=\bu^\top\bx$ and responds to it quadratically with curvature $\beta$. Since only $\beta\bu\bu^\top$ enters $f$, the scale of $\bu$ is degenerate with $\beta$ and we fix it by $\bu^\top\bSig\bu=1$, so that $t$ has unit within-class variance. For a fixed $\bu$, the read-out $\btheta:=(\bw,\beta,c)$ is trained by the population squared loss
\begin{equation}
  \btheta^\star=\argmin_{\btheta}\ \mathcal{L}(\btheta),\qquad \mathcal{L}(\btheta) = \tfrac{1}{2}\,\Ex\big[(f(\bx)-y)^2\big],
  \label{eq:loss}
\end{equation}
where the expectation is over the poisoned mixture \eqref{eq:poisoned}. The map $\btheta\mapsto f$ is linear, so $\mathcal L$ is a strictly convex quadratic and $\btheta^\star$ is unique. What distinguishes a model that has learned its representation from one that has not is the feature direction $\bu$, and we contrast the following two regimes.

\begin{definition}[Linear regime: no feature learning]
\label{def:linear}
The model is the linear classifier $f(\bx)=\bw^\top\bx+c$, i.e.\ $\beta\equiv0$ in \eqref{eq:neuron}, trained by \eqref{eq:loss}.
\end{definition}

This is the model implicitly assumed by lazy, kernel and random-feature analyses of poisoning: the representation is frozen and only a linear read-out is fitted. We show in Section~\ref{sec:linear} that it is exactly the special case $\bu^\top\bv=0$ of \eqref{eq:neuron}, and that a quadratic neuron whose direction is frozen at a random initialization reduces to it in high dimension.

\begin{definition}[Feature-learning regime]
\label{def:fl}
The feature direction is adapted to the poisoned data. Among all directions with $\bu^\top\bSig\bu=1$, the one that maximizes the feature--trigger overlap $\bu^\top\bv$ is the \emph{whitened trigger}
\begin{equation}
  \bu^\star = \frac{\bSig^{-1}\bv}{\sqrt{\bv^\top\bSig^{-1}\bv}},
  \qquad\text{since}\qquad
  |\bu^\top\bv| = |(\bSig^{1/2}\bu)^\top(\bSig^{-1/2}\bv)| \le \sqrt{\bv^\top\bSig^{-1}\bv}
  \label{eq:ustar}
\end{equation}
by the Cauchy--Schwarz inequality, with equality if and only if $\bu\propto\bSig^{-1}\bv$. The feature-learning regime is the model \eqref{eq:neuron} with $\bu=\bu^\star$ and the full read-out $(\bw,\beta,c)$ trained by \eqref{eq:loss}.
\end{definition}

Fixing $\bu=\bu^\star$ is an idealization of feature learning that keeps the analysis exact. It is supported by two facts: the poison switches on the curvature $\beta$ in proportion to $\bu^\top\bv$ (Section~\ref{sec:feature}), and when all of $(\bw,\beta,c,\bu)$ are trained by gradient descent the learned direction keeps an overlap $\bu^\top\bv$ between $0.6$ and $0.95$ of its maximal value (Appendix~\ref{app:endtoend}). Our general result of Section~\ref{sec:general} holds for every fixed $\bu$, so partially aligned directions are covered as well.

\begin{assumption}
\label{ass:model}
The within-class noise is Gaussian with a fixed covariance $\bSig\succ0$; the feature direction $\bu$ is fixed with $\bu^\top\bSig\bu=1$; the read-out $(\bw,\beta,c)$ is the population least-squares minimizer \eqref{eq:loss}. Gaussianity is the only assumption relaxed on real data in Section~\ref{sec:experiments}.
\end{assumption}

\section{Main Results}
\label{sec:main_results}

Since both the loss \eqref{eq:loss} and the sign of $f$ are invariant under the whitening $\tx=\bSig^{-1/2}\bx$, $\tw=\bSig^{1/2}\bw$, $\tu=\bSig^{1/2}\bu$ (Lemma~\ref{lem:whiten} in the appendix), the problem depends on $(\bmu,\bv,\bSig,\bu)$ only through the geometry of the whitened vectors $\tmu=\bSig^{-1/2}\bmu$, $\tv=\bSig^{-1/2}\bv$ and $\tu$, which is captured by five scalars:
\begin{equation}
  \underbrace{S = \sqrt{\bmu^\top\bSig^{-1}\bmu}}_{\text{signal strength}},
  \quad
  \underbrace{T = \sqrt{\bv^\top\bSig^{-1}\bv}}_{\text{trigger strength}},
  \quad
  \underbrace{\rho = \frac{\bmu^\top\bSig^{-1}\bv}{ST}}_{\text{signal--trigger alignment}},
  \quad
  \underbrace{g = \bu^\top\bmu}_{\text{feature--signal overlap}},
  \quad
  \underbrace{h = \bu^\top\bv}_{\text{feature--trigger overlap}}.
  \label{eq:invariants}
\end{equation}
$S$ and $T$ are the Mahalanobis norms of the class signal and of the trigger, i.e.\ their sizes in units of the within-class noise: $\Phi(S)$ is the accuracy of the optimal clean classifier, with $\Phi$ the standard normal distribution function, and $T$ is large when the trigger points along directions in which the data already vary, a trigger ``hidden in the noise''. $\rho\in[-1,1]$ is the cosine between signal and trigger in the whitened geometry, and $\rho=0$ (a trigger Mahalanobis-orthogonal to the signal) is the natural setting for a hidden trigger. Finally $|g|\le S$ and $|h|\le T$; the feature-learning direction \eqref{eq:ustar} has $h=T$ and $g=S\rho$, while the linear regime corresponds to $h=0$. All results below are explicit functions of $(S,T,\rho,g,h)$ and $(\pi,\alpha)$, and the dimension $d$ enters only through these scalars. We write $A:=\alpha T$ for the whitened trigger displacement, $\varphi$ for the standard normal density and $\Ex_z$ for the expectation over $z\sim\mathcal N(0,1)$.

\subsection{Exact Clean Accuracy and Attack Success Rate}
\label{sec:general}

The first step decomposes the $d$-dimensional problem to three coordinates: the feature, the part of the signal not explained by the feature, and the part of the trigger explained by neither.

\begin{lemma}[Reduction]
\label{lem:reduce}
Let $m_\perp=\sqrt{S^2-g^2}$, $v_\parallel=(ST\rho-gh)/m_\perp$ and $v_\perp=\sqrt{T^2-h^2-v_\parallel^2}$ (with $v_\parallel=0$ if $m_\perp=0$). There is an orthonormal frame $(\be_1,\be_2,\be_3)$ of whitened space with $\be_1=\tu$ such that the trained read-out satisfies $\tw\in\operatorname{span}(\be_1,\be_2,\be_3)$, and in the coordinates $(t,s_1,s_2)=(\be_1^\top\tx,\be_2^\top\tx,\be_3^\top\tx)$ the three clusters of \eqref{eq:poisoned} have means
\begin{equation}
  \bme_+ = (g,\ m_\perp,\ 0),\qquad
  \bme_- = -\bme_+,\qquad
  \bme_p = (g+\alpha h,\ m_\perp+\alpha v_\parallel,\ \alpha v_\perp),
  \label{eq:frame}
\end{equation}
with independent standard Gaussian noise in each coordinate, and $f = w_t\,t + w_1 s_1 + w_2 s_2 + \tfrac{\beta}{2}t^2 + c$ with $\btheta=(w_t,w_1,w_2,\beta,c)\in\R^5$. If $\bu\in\operatorname{span}(\bSig^{-1}\bmu,\bSig^{-1}\bv)$, which includes both regimes of Definitions~\ref{def:linear}--\ref{def:fl}, then $v_\perp=0$ and $w_2=0$.
\end{lemma}

In this frame the model is linear in the feature vector $\bphi=(t,s_1,s_2,\tfrac12t^2,1)$, so the trained read-out solves the normal equations $\bA\btheta=\bb$ with $\bA=\Ex[\bphi\bphi^\top]$ and $\bb=\Ex[\bphi\,y]$, whose entries are low-order Gaussian moments of the three clusters, hence explicit polynomials in the invariants and $(\pi,\alpha)$ (Appendix~\ref{app:moments}). The second step is that $f$ is quadratic in the noise only through the feature coordinate and linear in the two others, which can therefore be integrated out exactly.

\begin{theorem}[General closed form]
\label{thm:main}
Under Assumption~\ref{ass:model}, for any $\bSig\succ 0$, any $(\pi,\alpha)$, any alignment $\rho$ and any feature direction $\bu$, let $(w_t,w_1,w_2,\beta,c)=\bA^{-1}\bb$ and $\sigma=\sqrt{w_1^2+w_2^2}$. For a cluster mean $\bme=(m_t,m_1,m_2)$ define $A_0(\bme) = w_t m_t + w_1 m_1 + w_2 m_2 + c + \tfrac{\beta}{2}m_t^2$ and $b(\bme) = w_t + \beta\, m_t$. Then
\begin{align}
  \CACC &= \tfrac{1}{2}\,\Ex_z\!\left[
    \Phi\!\left(\frac{A_0(\bme_+) + b(\bme_+)z + \tfrac{\beta}{2}z^2}{\sigma}\right)
  \right]
  \notag + \tfrac{1}{2}\,\Ex_z\!\left[
    \Phi\!\left(-\,\frac{A_0(\bme_-) + b(\bme_-)z + \tfrac{\beta}{2}z^2}{\sigma}\right)
  \right],
  \notag\\
  \ASR &= \Ex_z\!\left[
    \Phi\!\left(-\,\frac{A_0(\bme_p) + b(\bme_p)z + \tfrac{\beta}{2}z^2}{\sigma}\right)
  \right],
  \label{eq:thm}
\end{align}
with the cluster means of \eqref{eq:frame} and the convention $\Phi(x/0)=\mathbf 1\{x>0\}$.
\end{theorem}

In essence, Theorem~\ref{thm:main} expresses each metric as the Gaussian mass on one side of a parabola: $A_0(\bme)$ is the value of $f$ at the cluster center, $b(\bme)z+\tfrac\beta2z^2$ is the fluctuation of the margin due to noise along the feature, which is quadratic because the neuron is, and $\sigma$ is the standard deviation of the margin due to noise orthogonal to the feature, which is integrated out by $\Phi$. When $\beta=0$ the parabola is a line and each expectation collapses to a single $\Phi$, the linear-classifier formula. Each metric is a one-dimensional integral evaluated by Gauss--Hermite quadrature from the five invariants. Two consequences hold for every feature direction.

\begin{corollary}
\label{cor:general}
For any $\bu$, any $\alpha$ and any $\rho$:
\begin{enumerate}[label=(\roman*),leftmargin=2em,itemsep=1pt,topsep=2pt]
\item \textbf{(Poison-free baseline.)} At $\pi=0$ the trained model is the linear classifier $\tw=\tmu/(1+S^2)$, $\beta=c=0$; hence $\CACC=\Phi(S)$ and $\ASR=\Phi(-S-\alpha T\rho)$.
\item \textbf{(Stealth.)} $\CACC=\Phi(S)+O(\pi^2)$.
\end{enumerate}
\end{corollary}

Item (i) shows that a trigger that is not Mahalanobis-orthogonal to the signal ($\rho\neq0$) is already seen by the clean classifier, and for $\rho<0$ it drives the ASR to one at no poisoning cost; this is not a backdoor, and the asymptotic laws below are stated for the hidden trigger $\rho=0$, for which the baseline ASR is $\Phi(-S)$. Item (ii) is the stealth property: the $O(\pi)$ change of the read-out shifts the margins of the two clean classes by opposite amounts (Appendix~\ref{app:stealth}), so the attack has no first-order accuracy signature in either regime.

\subsection{The Linear Regime: No Feature Learning}
\label{sec:linear}

We first specialize Theorem~\ref{thm:main} to a model that does not learn its features, the regime behind lazy-training budget estimates.

\begin{proposition}
\label{prop:linear}
\emph{(i)} If $h=\bu^\top\bv=0$, then $\beta=0$ exactly for all $(\pi,\alpha)$, and $(\bw,c)$ coincide with the least-squares linear classifier of Definition~\ref{def:linear}. \emph{(ii)} If $\bu$ is frozen at a random direction ($\tu$ uniform on the unit sphere of $\R^d$), then $\Ex[h^2]=T^2/d$; together with $\beta=O(\pi h)$ (Proposition~\ref{prop:beta}), a lazy quadratic neuron in high dimension is the linear classifier up to $O(d^{-1/2})$.
\end{proposition}

Part (i) states that a feature orthogonal to the trigger cannot detect it: with $h=0$ the quadratic feature $\tfrac12t^2$ has the same distribution in all three clusters and least squares assigns it zero weight. Part (ii) is the blindness of a frozen random feature in high dimension. In the linear regime the trained weights and the metrics have explicit single-$\Phi$ expressions (Corollary~\ref{cor:linear} in Appendix~\ref{app:linear}), from which we obtain the budget curve.

\begin{theorem}[Linear regime]
\label{thm:linear}
In the linear regime with a hidden trigger ($\rho=0$):
\begin{enumerate}[label=(\alph*),leftmargin=2em,itemsep=1pt,topsep=2pt]
\item \textbf{(Small poison.)} $\CACC=\Phi(S)+O(\pi^2)$ and
  $\ASR = \Phi(-S) + \pi\,\varphi(S)\Bigl[\dfrac{(1+2S^2)\,A^2}{2S} + \dfrac{1+S^2}{S}\Bigr] + O(\pi^2)$.
\item \textbf{(Stealth limit.)} As $\pi\to0$ and $\alpha\to\infty$ with the trigger energy $e:=\pi\alpha^2T^2$ fixed,
\begin{equation}
  \CACC\to\Phi(S),
  \qquad
  \ASR \to \ASR_{\rm lin}(e) := \Phi\!\left(-S + \frac{1+2S^2}{S}\cdot\frac{e}{e+2}\right),
  \label{eq:linlaw}
\end{equation}
which increases from $\Phi(-S)$ at $e=0$ to the ceiling $\Phi\!\bigl(\tfrac{1+S^2}{S}\bigr)<1$ as $e\to\infty$.
\item \textbf{(Budget curve.)} For a target $a<\Phi\!\bigl(\tfrac{1+S^2}{S}\bigr)$, writing $\kappa_a:=\dfrac{S\,(S+\Phi^{-1}(a))}{1+2S^2}\in[0,1)$,
\begin{equation}
  \alpha_{\rm lin}(\pi;a) = \frac{1}{T}\sqrt{\frac{e_a}{\pi}}\ \propto\ \pi^{-1/2},
  \qquad
  e_a = \frac{2\kappa_a}{1-\kappa_a}.
  \label{eq:lincost}
\end{equation}
\end{enumerate}
\end{theorem}

Essentially, Theorem~\ref{thm:linear} states that a model without feature learning sees the trigger only through the mean shift of the poisoned cluster: the read-out tilts toward $-\tv$ by an amount proportional to $\pi\alpha$, and a triggered test point, displaced by the same $\alpha\tv$, is pushed across the boundary by an amount proportional to $\pi\alpha^2$. Specifically, we draw the following observations:
\begin{itemize}[leftmargin=2em,itemsep=1pt,topsep=2pt]
\item \emph{Effect of the poison fraction $\pi$:} the only combination that survives the stealth limit is the trigger energy $e=\pi\alpha^2T^2$ (recall that $A=\alpha T$), and the budget curve scales as $\pi^{-1/2}$.
\item \emph{Effect of the trigger geometry:} $\alpha_{\rm lin}\propto1/T$, so a trigger along a direction of large within-class variance is cheaper, because in whitened units it is already large.
\item \emph{Effect of the target $a$:} the ASR saturates at $\Phi\!\bigl(\tfrac{1+S^2}{S}\bigr)<1$: the clean classes pin the read-out along $\tmu$, and the residual noise along $\tmu$ lets a fraction $\Phi\!\bigl(-\tfrac{1+S^2}{S}\bigr)$ of triggered points escape whatever the trigger strength.
\end{itemize}

\subsection{The Feature-Learning Regime}
\label{sec:feature}

We now let the neuron use its quadratic feature. The first result identifies what the poison does to the read-out for an arbitrary feature direction.

\begin{proposition}[The poison creates a trigger detector]
\label{prop:beta}
For any feature direction $\bu$, expanding $\btheta^\star=\bA^{-1}\bb$ to first order in $\pi$ gives $\tw=\tmu/(1+S^2)+O(\pi)$, $c=O(\pi)$ and
\begin{equation}
  \beta = -\,\frac{\pi\,\alpha h\,
    \bigl[ST\rho\,h\,\alpha^2
      + \bigl((2S^2+1)h + 2ST\rho\,g\bigr)\alpha + 4g(S^2+1)\bigr]}
    {2(S^2+1)(2g^2+1)} + O(\pi^2).
  \label{eq:beta}
\end{equation}
In the feature-learning regime with a hidden trigger ($h=T$, $g=0$, $\rho=0$) this reduces to $\beta = -\pi\,\tfrac{1+2S^2}{2(1+S^2)}\,A^2+O(\pi^2)$.
\end{proposition}

Three facts follow from \eqref{eq:beta}. First, $\beta=O(\pi)$ vanishes at $\pi=0$: the clean model has no quadratic unit (Corollary~\ref{cor:general}) and it is the poison that switches one on; its sign is negative, so the unit fires on large $|t|$ and routes such inputs to the target class. Second, $\beta\propto h=\bu^\top\bv$: this is why the linear regime ($h=0$) has no detector, why a frozen random feature ($h\sim d^{-1/2}$) loses it in high dimension, and why feature learning should select the direction that maximizes $h$. Third, in the feature-learning regime $\beta\propto\pi A^2$ grows with the square of the trigger displacement, whereas the linear read-out grows linearly, so on a triggered point, whose feature coordinate is $t\approx A$, the detector contributes $\tfrac\beta2A^2\propto\pi A^4$ to the margin. The exact finite-$\pi$ weights (Corollary~\ref{cor:flweights} in Appendix~\ref{app:fl}) show that their common denominator $\Delta=16\det\bA$ grows as $\pi A^4$, so $\beta=O(A^{-2})$ as $A\to\infty$ (the detector saturates) and $\pi A^4$ is the only combination that stays finite when $\pi\to0$ and $A\to\infty$.

\begin{figure}[htbp]
  \centering
  \includegraphics[width=0.6\linewidth]{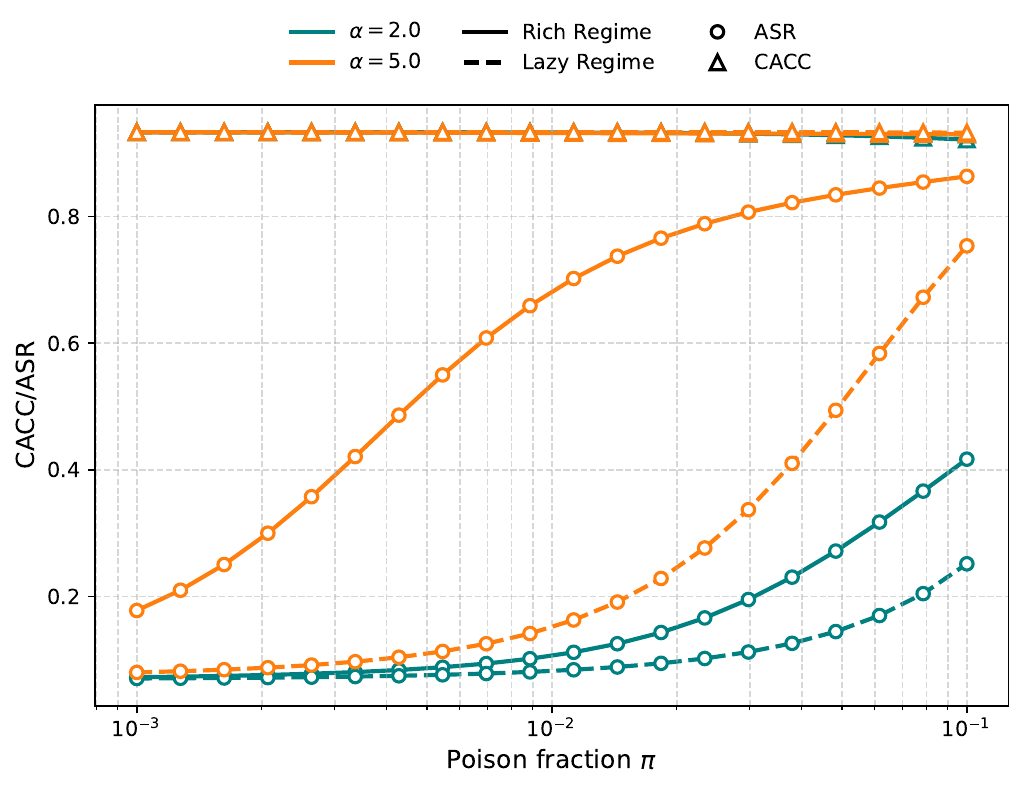}
  \caption{\textbf{Exact theory in the two regimes.} \emph{$\CACC$ (triangles) and $\ASR$ (circles) from Theorem~\ref{thm:main} as functions of $\pi$, for $\alpha=2$ and $\alpha=5$, in the feature-learning regime ($\bu=\bu^\star$, solid) and the linear regime ($\beta\equiv0$, dashed); $S=1.5$, $T=1$, $\rho=0$. In both regimes the clean accuracy stays at $\Phi(S)\approx0.93$, while the attack success rate is uniformly higher with feature learning and, at $\alpha=5$, exceeds $0.5$ at a poison fraction roughly ten times smaller than in the linear regime, consistent with \eqref{eq:ratio}.}}
  \label{fig:theory_asr_cacc}
\end{figure}

\begin{theorem}[Feature-learning regime]
\label{thm:fl}
In the feature-learning regime with a hidden trigger ($\bu=\bu^\star$, $\rho=0$):
\begin{enumerate}[label=(\alph*),leftmargin=2em,itemsep=1pt,topsep=2pt]
\item \textbf{(Small poison.)} $\CACC=\Phi(S)+O(\pi^2)$ and
\begin{equation}
  \ASR = \Phi(-S)
        + \pi\,\varphi(S)\left[
          \frac{(1+2S^2)\,A^2(A^2+2)}{4S}
          + \frac{1+S^2}{S}
          \right] + O(\pi^2),
  \label{eq:asr_expand}
\end{equation}
which exceeds the expansion of Theorem~\ref{thm:linear}(a) by exactly $\pi\varphi(S)\,(1+2S^2)A^4/(4S)$.
\item \textbf{(Stealth limit.)} As $\pi\to0$ and $\alpha\to\infty$ with the quartic trigger energy $\varrho:=\pi\alpha^4T^4$ fixed,
\begin{equation}
  \CACC\to\Phi(S),
  \qquad
  \ASR\to\ASR_{\rm fl}(\varrho) := \Phi\!\left(-S + \frac{1+2S^2}{S}\cdot\frac{\varrho}{4+\varrho}\right).
  \label{eq:fllaw}
\end{equation}
\item \textbf{(Ceiling.)} At any finite $\pi$ and for any alignment $\rho$, $\displaystyle\lim_{\alpha\to\infty}\ASR = \Phi\!\Bigl(\frac{1+S^2}{S} + \frac{\pi}{2S(1-\pi)}\Bigr) < 1$, independent of $T$ and $\rho$.
\item \textbf{(Budget curve.)} For a target $a<\Phi\!\bigl(\tfrac{1+S^2}{S}\bigr)$, with $\kappa_a$ as in Theorem~\ref{thm:linear},
\begin{equation}
  \alpha_{\rm fl}(\pi;a)
  = \frac{1}{T}\left(\frac{\varrho_a}{\pi}\right)^{1/4}\ \propto\ \pi^{-1/4},
  \qquad
  \varrho_a = \frac{4\kappa_a}{1-\kappa_a}.
  \label{eq:cost}
\end{equation}
\end{enumerate}
\end{theorem}

\begin{table}[t]
\centering
\caption{Linear regime versus feature-learning regime (hidden trigger, $\rho=0$; $A=\alpha T$; $\kappa_a=S(S+\Phi^{-1}(a))/(1+2S^2)$). The two regimes share the clean accuracy, the form of the ASR law and the stealth ceiling; they differ in the conserved trigger energy and hence in the cost exponent.}
\label{tab:compare}
\small
\begin{tabular}{lll}
\toprule
 & \textbf{Linear regime} (Thm.~\ref{thm:linear}) & \textbf{Feature learning} (Thm.~\ref{thm:fl})\\
\midrule
Trained curvature & $\beta\equiv0$ & $\beta=-\pi\frac{1+2S^2}{2(1+S^2)}A^2+O(\pi^2)$\\
Trigger term in the margin & $\propto\pi A^2$ & $\tfrac{\beta}{2}A^2\propto\pi A^4$\\
Conserved invariant & $e=\pi\alpha^2T^2$ & $\varrho=\pi\alpha^4T^4$\\
Stealth-limit ASR & $\Phi\!\bigl(-S+\tfrac{1+2S^2}{S}\tfrac{e}{e+2}\bigr)$ & $\Phi\!\bigl(-S+\tfrac{1+2S^2}{S}\tfrac{\varrho}{4+\varrho}\bigr)$\\
Ceiling ($\pi\to0$) & $\Phi\!\bigl(\tfrac{1+S^2}{S}\bigr)$ & $\Phi\!\bigl(\tfrac{1+S^2}{S}\bigr)$\\
Budget curve $\alpha(\pi;a)$ & $\tfrac1T\bigl(\tfrac{2\kappa_a}{(1-\kappa_a)\pi}\bigr)^{1/2}$ & $\tfrac1T\bigl(\tfrac{4\kappa_a}{(1-\kappa_a)\pi}\bigr)^{1/4}$\\
Clean accuracy & $\Phi(S)+O(\pi^2)$ & $\Phi(S)+O(\pi^2)$\\
\bottomrule
\end{tabular}
\end{table}

Table~\ref{tab:compare} places the two regimes side by side and Figure~\ref{fig:theory_asr_cacc} plots the exact curves of Theorem~\ref{thm:main} in both. Particularly, Theorem~\ref{thm:fl} shows that the two regimes obey the same law in a different currency. The stealth-limit laws \eqref{eq:linlaw} and \eqref{eq:fllaw} have an identical form: the margin argument moves from the clean error level $-S$ toward the common ceiling $\tfrac{1+S^2}{S}$ by the same total budget $\tfrac{1+2S^2}{S}$, multiplied by a saturating fraction in $[0,1)$. What differs is the variable inside that fraction: the linear model spends the quadratic energy $e=\pi\alpha^2T^2$, the feature-learning model spends the quartic energy $\varrho=\pi\alpha^4T^4$. The reason is visible in the small-poison expansions, which agree except for the term $\pi\varphi(S)(1+2S^2)A^4/(4S)$ in \eqref{eq:asr_expand}: the detector responds to the trigger as $(\bu^{\star\top}\bx)^2\sim A^2$ and least squares scales its curvature by another $\pi A^2$. Specifically, we draw the following observations:
\begin{itemize}[leftmargin=2em,itemsep=1pt,topsep=2pt]
\item \emph{Effect of the poison fraction $\pi$:} conservation of $\varrho$ instead of $e$ halves the cost exponent, from $\pi^{-1/2}$ to $\pi^{-1/4}$. Dividing \eqref{eq:lincost} by \eqref{eq:cost}, the advantage at a fixed target is
\begin{equation}
  \frac{\alpha_{\rm lin}(\pi;a)}{\alpha_{\rm fl}(\pi;a)}
  = \left(\frac{\kappa_a}{(1-\kappa_a)\,\pi}\right)^{1/4},
  \label{eq:ratio}
\end{equation}
which grows without bound as the attack becomes stealthier.
\item \emph{Effect of the target $a$:} feature learning does not raise the ceiling, it reaches it at a lower cost. In the stealth limit both regimes saturate at $\Phi\!\bigl(\tfrac{1+S^2}{S}\bigr)$, and at finite $\pi$ the ceiling of Theorem~\ref{thm:fl}(c) exceeds it only by $\pi/(2S(1-\pi))$: as $A\to\infty$ the curvature decays as $\beta=O(A^{-2})$, the triggered statistic becomes linear-Gaussian again, and the ceiling is set by the clean classes as in the linear regime.
\item \emph{Effect of the geometry $(T,\rho)$ and on clean accuracy:} the ceiling depends on neither $T$ nor $\rho$, in both regimes $\alpha\propto1/T$, and $\CACC=\Phi(S)+O(\pi^2)$, so the amplification comes at no cost in clean accuracy and leaves no first-order signature for a defender to detect.
\end{itemize}
We point out that an auditor who models the adversary with a lazy learner concludes that a poison fraction $\pi$ needs a trigger of strength $\alpha_{\rm lin}(\pi;a)$; if the audited model has learned its features, the true requirement is smaller by the factor \eqref{eq:ratio}, an order of magnitude at the poison fractions of large-scale training. This is the theoretical counterpart of the observation of \citet{souly2025poisoning} that the number of poisoned documents needed to backdoor a language model does not grow with the corpus.

\begin{figure}[t]
  \centering
  \includegraphics[width=\linewidth]{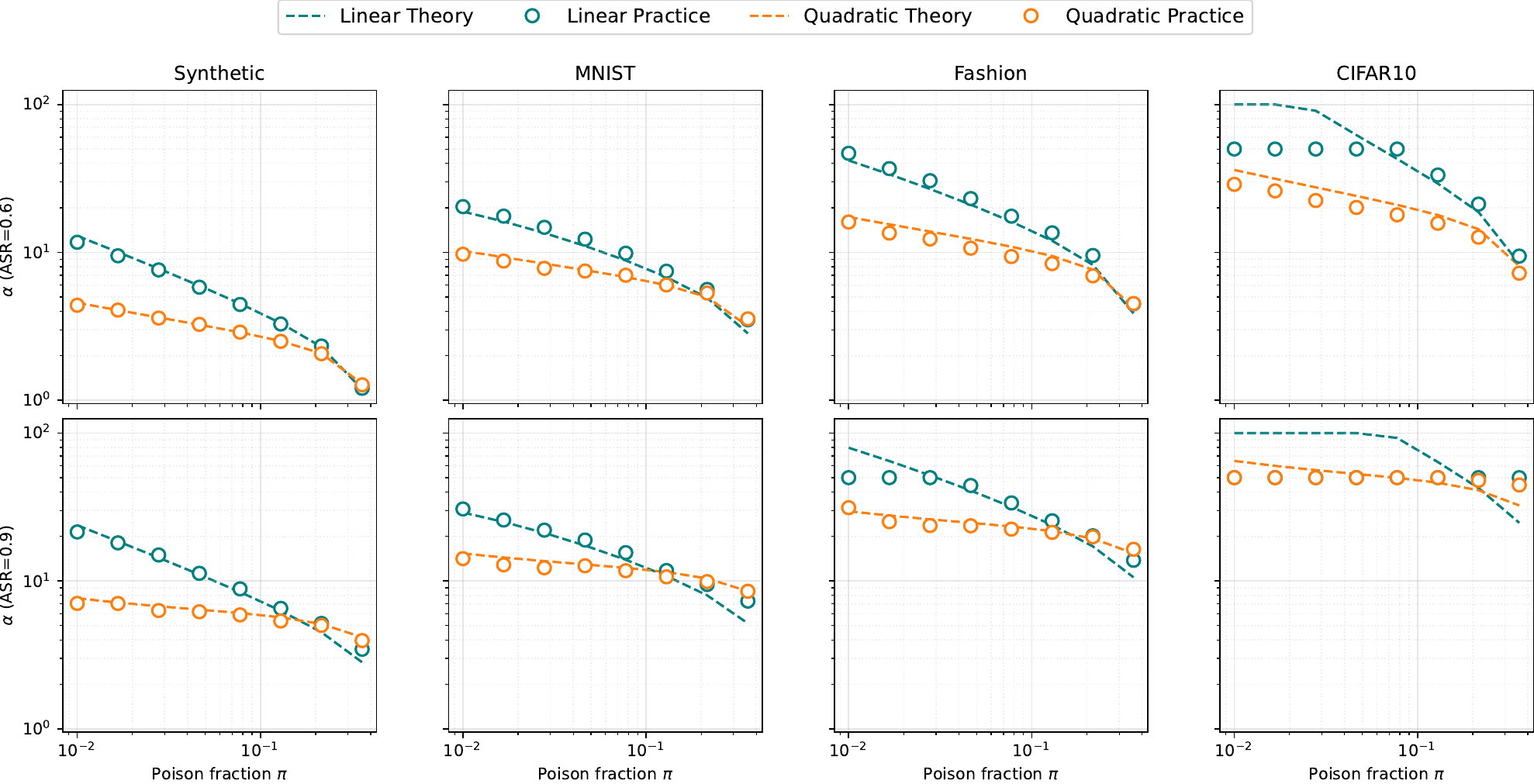}
  \caption{\textbf{Theory versus practice on synthetic and real data.} \emph{Trigger strength $\alpha$ required to reach $\ASR=0.6$ (top) and $\ASR=0.9$ (bottom) as a function of $\pi$. Markers are trained models (feature-learning regime in red, linear regime in blue) and lines the predictions of Theorem~\ref{thm:main} on the invariants extracted from the data. The quadratic neuron follows the $\alpha\propto\pi^{-1/4}$ law, the linear model the steeper $\alpha\propto\pi^{-1/2}$ law.}}
  \label{fig:scaling_datasets}
\end{figure}

\subsection{Partially Learned Features}
\label{sec:channels}

Theorem~\ref{thm:main} also covers models that have learned their features only partially, $0<h<T$. The trigger then splits into a component $\alpha h$ along the feature, which drives the quadratic detector (invariant $\pi(\alpha h)^4$), and a component $\alpha v_\perp$, $v_\perp:=\sqrt{T^2-h^2}$, orthogonal to it, which drives the linear channel of Section~\ref{sec:linear} (invariant $\pi\alpha^2v_\perp^2$); the attacker pays whichever is cheaper.

\begin{corollary}[Two channels]
\label{cor:channels}
Let $\rho=0$. The feature-learning channel ($\alpha\propto\pi^{-1/4}$) dominates for $h\gtrsim T\pi^{1/4}$ and the linear channel ($\alpha\propto\pi^{-1/2}$) below it. In the stealth regime the two laws combine into the generalized energy $G=2\pi\alpha^2v_\perp^2+\pi\alpha^4h^4$, $\ASR\simeq\Phi\bigl(-S+\tfrac{1+2S^2}{S}\tfrac{G}{4+G}\bigr)$, whose inversion (Appendix~\ref{app:channels}) gives a closed-form $\alpha(\pi;a,h)$ that is exact at both ends, $h\to0$ and $h=T$.
\end{corollary}

Figure~\ref{fig:phase} in Appendix~\ref{app:channels} shows the resulting phase diagram. The crossover $h=T\pi^{1/4}$ moves to smaller overlaps as the poison becomes rarer, so an auditor who cannot inspect the learned feature direction must assume that the cheaper channel is active.

\section{Experiments}
\label{sec:experiments}

\textbf{Setup.} We evaluate the theory on a synthetic Gaussian mixture and on three image benchmarks, MNIST (3 vs.\ 8), Fashion-MNIST (Pullover vs.\ Coat) and CIFAR-10 (Ship vs.\ Bird). For the image datasets we estimate the class means $\pm\hat{\bmu}$ and the pooled within-class covariance $\hat{\bSig}$ (with a ridge $10^{-4}\mI$) and extract the invariants \eqref{eq:invariants}, so that all predictions are parameter-free. To simulate a stealthy adversary, the trigger $\bv$ is the principal eigenvector of $\hat{\bSig}$, the direction of maximal within-class variance, and the feature direction is $\bu^\star=\hat{\bSig}^{-1}\bv/T$. We compare the \emph{feature-learning regime} (trained $\beta$) and the \emph{linear regime} ($\beta=0$). \emph{Practice} markers are obtained by sampling $n_{\rm train}=10{,}000$ points from the Gaussian mixture with parameters $(\hat{\bmu},\hat{\bSig})$, poisoning a fraction $\pi$ of the source class, solving the finite-sample least-squares problem and measuring the ASR on $n_{\rm test}=5{,}000$ held-out samples; \emph{Theory} lines are evaluated from Theorem~\ref{thm:main}. Empirically, for each $\pi$ on a logarithmic grid from $10^{-2}$ to $0.36$ we find by binary search the trigger strength that reaches $\ASR=0.6$ and $\ASR=0.9$, and fit $\log\alpha=k\log\pi+b$ to the resulting budget curves.

\textbf{Results.} Figure~\ref{fig:scaling_datasets} shows the budget curves. On all four datasets the empirical quadratic neuron tracks the theoretical curve, and the fitted slopes are $k\approx-\tfrac14$ in the feature-learning regime and $k\approx-\tfrac12$ in the linear regime, as predicted by Theorems~\ref{thm:linear} and~\ref{thm:fl}; the gap between the two regimes widens as $\pi$ decreases, in agreement with \eqref{eq:ratio}. Table~\ref{tab:cacc_asr_results} in Appendix~\ref{app:exp} reports $\CACC$ and $\ASR$ on a grid of $(\pi,\alpha)$: in both regimes the clean accuracy is unchanged by the poison, while the attack success rate is uniformly larger with feature learning (on MNIST, $0.97$ against $0.15$ at $\pi=0.01$, $\alpha=1$). Finally, Figure~\ref{fig:mlp} shows that the separation is not specific to the quadratic neuron: a network trained in the rich feature learning regime requires a smaller trigger than the same network trained in the lazy regime.

\section{Discussions \& Conclusion}
\label{sec:conclusion}

In this paper, we studied backdoor poisoning through a statistical model that isolates the role of feature learning: a quadratic neuron trained by least squares on a poisoned Gaussian mixture. We derived exact expressions for the clean accuracy and the attack success rate, valid for arbitrary noise covariance, finite poison fraction and trigger strength, and used them to show that a model without feature learning obeys the budget curve $\alpha\propto\pi^{-1/2}$, whereas a model that learns its feature obeys $\alpha\propto\pi^{-1/4}$: the poison creates a quadratic trigger detector whose response scales as $\alpha^4$, which halves the cost exponent without raising the attack ceiling and without affecting the clean accuracy. Our results offer a mechanism for the observation that the number of poisoned samples needed to backdoor a large model does not grow with its training corpus, and imply that audits calibrated on lazy-regime estimates underestimate the adversary by a factor that grows as $\pi^{-1/4}$ in our considered setting. Extending the analysis to a hidden layer with a general activation, to finite-sample training and to trained feature directions is left for future work, and could potentially bring more understanding on the inner interaction between backdoor attacks and neural feature learning.

\section*{Acknowledgements}
We thank Sara Tucci-Piergiovanni, Sami Souihi, and Mohamed Tamaazousti for helpful discussions. The computational work presented in this paper was performed using the CEA List FactoryIA supercomputer, with financial support from the Île-de-France Regional Council.

\bibliography{references}

\appendix
\section{Proof of Theorem~\ref{thm:main} and of Corollary~\ref{cor:general}}
\label{app:general}

Throughout, $\Phi$ and $\varphi$ are the standard normal distribution and density, and $\Ex_z$ denotes expectation over $z\sim\mathcal N(0,1)$.

\subsection{Whitening}
\label{app:whiten}

\begin{lemma}[Whitening equivariance]
\label{lem:whiten}
Let $\tx=\bSig^{-1/2}\bx$, $\tw=\bSig^{1/2}\bw$, $\tu=\bSig^{1/2}\bu$. Then $f(\bx)=\tw^\top\tx+\tfrac\beta2(\tu^\top\tx)^2+c$ for all $\bx$, the whitened training distribution is the mixture \eqref{eq:poisoned} with $(\bmu,\bv,\bSig)$ replaced by $(\tmu,\tv,\mI)$, $\tmu=\bSig^{-1/2}\bmu$, $\tv=\bSig^{-1/2}\bv$, and $\|\tu\|=1$. The population loss \eqref{eq:loss}, its minimiser (expressed in the tilded variables), $\CACC$ and $\ASR$ are all unchanged. Moreover $S=\|\tmu\|$, $T=\|\tv\|$, $\rho=\tmu^\top\tv/(ST)$, $g=\tu^\top\tmu$ and $h=\tu^\top\tv$.
\end{lemma}

\begin{proof}
$\bw^\top\bx=(\bSig^{1/2}\bw)^\top(\bSig^{-1/2}\bx)=\tw^\top\tx$ and likewise $\bu^\top\bx=\tu^\top\tx$, which gives the identity for $f$. If $\bx=\bme_c+\bz$ with $\bz\sim\mathcal N(\bzero,\bSig)$ then $\tx=\bSig^{-1/2}\bme_c+\tz$ with $\tz\sim\mathcal N(\bzero,\mI)$, and $\bSig^{-1/2}\bme_c\in\{\tmu,-\tmu,\tmu+\alpha\tv\}$. The map $\bw\mapsto\tw$ is a bijection, so minimising the loss over $\bw$ or over $\tw$ is the same problem, and the sign of $f$ is pointwise unchanged, hence the metrics are. Finally $\|\tu\|^2=\bu^\top\bSig\bu=1$, $\|\tmu\|^2=\bmu^\top\bSig^{-1}\bmu=S^2$, $\|\tv\|^2=T^2$, $\tmu^\top\tv=\bmu^\top\bSig^{-1}\bv=ST\rho$, and $\tu^\top\tmu=(\bSig^{1/2}\bu)^\top(\bSig^{-1/2}\bmu)=\bu^\top\bmu=g$, $\tu^\top\tv=\bu^\top\bv=h$.
\end{proof}

By Lemma~\ref{lem:whiten} we work in whitened coordinates from now on and drop the tildes: the noise is $\mathcal N(\bzero,\mI)$, the signal is $\tmu$ with $\|\tmu\|=S$, the trigger is $\tv$ with $\|\tv\|=T$, the feature direction $\tu$ is a unit vector, and the three cluster means are $\tmu$, $-\tmu$, $\tmu+\alpha\tv$.

\subsection{Proof of Lemma~\ref{lem:reduce} (reduction to a three-dimensional frame)}
\label{app:reduce}

\paragraph{The read-out lies in the span of signal, trigger and feature.}
Let $W=\operatorname{span}(\tmu,\tv,\tu)$ and let $\bq\in W^\perp$ be arbitrary. For fixed $(\beta,c)$ the stationarity condition of \eqref{eq:loss} with respect to $\tw$ reads
\begin{equation}
  \Ex\bigl[\tx\,(\tw^\top\tx+\tfrac\beta2 t^2+c-y)\bigr]=\bzero,
  \qquad t=\tu^\top\tx .
  \label{eq:statw}
\end{equation}
Take the inner product of \eqref{eq:statw} with $\bq$. Write $\tx=\bme_c+\tz$ for the cluster $c$ of the sample. Since $\bq\perp\bme_c$, $\bq^\top\tx=\bq^\top\tz$, and since $\bq\perp\tu$, the Gaussian variables $\bq^\top\tz$ and $\tu^\top\tz$ are uncorrelated hence independent. The bracket in \eqref{eq:statw} depends on the sample only through $(c,\ \tw^\top\tz,\ \tu^\top\tz)$; decompose $\tw=\tw_W+\tw_\perp$ with $\tw_W\in W$, $\tw_\perp\in W^\perp$. Then
\begin{align*}
\bq^\top\Ex\bigl[\tx(\cdots)\bigr]
&=\Ex\bigl[\bq^\top\tz\,(\tw_\perp^\top\tz)\bigr]
+\Ex\bigl[\bq^\top\tz\,\underbrace{(\tw_W^\top\tx+\tfrac\beta2t^2+c-y)}_{\text{independent of }\bq^\top\tz}\bigr]\\
&=\bq^\top\tw_\perp+0 ,
\end{align*}
because $\Ex[\tz\tz^\top]=\mI$ and $\Ex[\bq^\top\tz]=0$. Hence $\bq^\top\tw_\perp=0$ for every $\bq\in W^\perp$, i.e.\ $\tw_\perp=\bzero$ and $\tw\in W$. Consequently $f$ depends on $\tx$ only through its orthogonal projection onto $W$, and minimising \eqref{eq:loss} over $\tw\in\R^d$ is the same as minimising over $\tw\in W$.

\paragraph{The frame.}
Let $\be_1=\tu$. If $m_\perp=\sqrt{S^2-g^2}>0$, let $\be_2=(\tmu-g\tu)/m_\perp$; this is a unit vector orthogonal to $\be_1$ because $\tu^\top\tmu=g$ and $\|\tmu-g\tu\|^2=S^2-2g^2+g^2=m_\perp^2$. Let $\be_3$ be a unit vector in $W$ orthogonal to $\be_1,\be_2$ (if $\dim W<3$, complete with any unit vector orthogonal to $W$; the corresponding coordinate then has zero mean in every cluster and, by the argument above applied to $W$, zero weight). In this frame $\tmu=g\be_1+m_\perp\be_2$ by construction. Writing $\tv=h\be_1+v_\parallel\be_2+v_\perp\be_3$, we have $h=\tu^\top\tv$, $v_\parallel=\be_2^\top\tv=(\tmu^\top\tv-g\,\tu^\top\tv)/m_\perp=(ST\rho-gh)/m_\perp$, and $v_\perp^2=\|\tv\|^2-h^2-v_\parallel^2=T^2-h^2-v_\parallel^2$; choosing the orientation of $\be_3$ makes $v_\perp\ge0$. If $m_\perp=0$ then $\tmu=\pm S\tu$ and we take $\be_2,\be_3$ any orthonormal completion with $\be_3\propto\tv-h\tu$, so that $v_\parallel=0$ and $v_\perp=\sqrt{T^2-h^2}$. The cluster means \eqref{eq:frame} follow by reading $\tmu$, $-\tmu$ and $\tmu+\alpha\tv$ in this basis. Since the frame is orthonormal and the noise is $\mathcal N(\bzero,\mI)$, the three noise coordinates are independent standard Gaussians. Finally $f=\tw^\top\tx+\tfrac\beta2 t^2+c$ with $\tw=w_t\be_1+w_1\be_2+w_2\be_3$ is Lemma~\ref{lem:reduce}.

\paragraph{The in-plane case.}
If $\bu\in\operatorname{span}(\bSig^{-1}\bmu,\bSig^{-1}\bv)$ then $\tu=\bSig^{1/2}\bu\in\operatorname{span}(\tmu,\tv)$, so $W$ is two-dimensional, $\be_3\perp W$, $v_\perp=0$, and $w_2=0$ by the projection argument. Both regimes are in-plane: the linear regime does not use $\bu$ at all (equivalently any $\bu$ with $h=0$ may be used, see Appendix~\ref{app:linear}), and $\bu^\star\propto\bSig^{-1}\bv$. \hfill$\square$

\subsection{Population moments}
\label{app:moments}

In the frame of Lemma~\ref{lem:reduce} the model is linear in $\bphi=(t,\,s_1,\,s_2,\,\tfrac12 t^2,\,1)^\top$, so the minimiser of \eqref{eq:loss} solves
\begin{equation}
  \bA\,\btheta = \bb,
  \qquad
  \bA = \Ex[\bphi\bphi^\top] = \sum_{c\in\{+,-,p\}}\omega_c\,\bM(\bme_c),
  \qquad
  \bb = \Ex[\bphi\,y] = \sum_{c\in\{+,-,p\}}\omega_c\,y_c\,\bar{\bphi}(\bme_c),
  \label{eq:normal}
\end{equation}
where, for a cluster with mean $\bme=(m_t,m_1,m_2)$ and unit isotropic noise,
\begingroup\small\setlength{\arraycolsep}{3pt}
\begin{equation}
  \bar{\bphi}(\bme)=
  \begin{pmatrix} m_t\\ m_1\\ m_2\\ \tfrac12(m_t^2+1)\\ 1\end{pmatrix},
  \quad
  \bM(\bme)=
  \begin{pmatrix}
    m_t^2+1 & m_t m_1 & m_t m_2 & \tfrac12(m_t^3+3m_t) & m_t\\
    \cdot & m_1^2+1 & m_1 m_2 & \tfrac12 m_1(m_t^2+1) & m_1\\
    \cdot & \cdot & m_2^2+1 & \tfrac12 m_2(m_t^2+1) & m_2\\
    \cdot & \cdot & \cdot & \tfrac14(m_t^4+6m_t^2+3) & \tfrac12(m_t^2+1)\\
    \cdot & \cdot & \cdot & \cdot & 1
  \end{pmatrix}
  \label{eq:moments}
\end{equation}
\endgroup
(symmetric). Indeed, for a single cluster with mean $\bme=(m_t,m_1,m_2)$ and independent unit-variance coordinates, the moments of $\bphi$ follow from $\Ex z=\Ex z^3=0$, $\Ex z^2=1$, $\Ex z^4=3$ for $z\sim\mathcal N(0,1)$:
\begin{gather*}
\Ex t^2=m_t^2+1,\qquad \Ex t^3=m_t^3+3m_t,\qquad \Ex t^4=m_t^4+6m_t^2+3,\\
\Ex[t s_j]=m_t m_j,\qquad \Ex[s_j\tfrac12t^2]=\tfrac12 m_j(m_t^2+1),\qquad \Ex[s_1 s_2]=m_1m_2,
\end{gather*}
which gives $\bar{\bphi}(\bme)$ and $\bM(\bme)$ in \eqref{eq:moments}. Summing over the three clusters with weights $\omega_c$ and labels $y_c$ gives $\bA$ and $\bb$ in \eqref{eq:normal}. For a monomial $q(\bme)$ write $\langle q\rangle:=\sum_c\omega_c q(\bme_c)$; because the clean means are $\pm\bme_+$, an even monomial in $\bme$ contributes $\tfrac{2-\pi}{2}q(\bme_+)+\tfrac\pi2 q(\bme_p)$ and an odd one $-\tfrac\pi2 q(\bme_+)+\tfrac\pi2 q(\bme_p)$. Writing $\nu:=g+\alpha h$ for the feature mean of the poisoned cluster, the sums needed are
\begingroup\small
\begin{align*}
\langle m_t\rangle&=\tfrac{\pi\alpha h}{2},\qquad
\langle m_1\rangle=\tfrac{\pi\alpha v_\parallel}{2},\qquad
\langle m_2\rangle=\tfrac{\pi\alpha v_\perp}{2},\qquad
\langle m_t^2\rangle=g^2+\tfrac{\pi}{2}(\alpha^2h^2+2\alpha g h),\\
\langle m_1^2\rangle&=m_\perp^2+\tfrac{\pi}{2}(\alpha^2v_\parallel^2+2\alpha m_\perp v_\parallel),\qquad
\langle m_2^2\rangle=\tfrac{\pi}{2}\alpha^2v_\perp^2,\qquad
\langle m_t m_2\rangle=\tfrac{\pi}{2}\alpha v_\perp\nu,\\
\langle m_tm_1\rangle&=g m_\perp+\tfrac{\pi}{2}(\alpha^2h v_\parallel+\alpha g v_\parallel+\alpha h m_\perp),\qquad
\langle m_1m_2\rangle=\tfrac{\pi}{2}\alpha v_\perp(m_\perp+\alpha v_\parallel),\\
\langle m_t^3\rangle&=\tfrac{\pi\alpha h}{2}(\alpha^2h^2+3\alpha g h+3g^2),\qquad
\langle m_2 m_t^2\rangle=\tfrac{\pi}{2}\alpha v_\perp\nu^2,\\
\langle m_t^4\rangle&=g^4+\tfrac{\pi\alpha h}{2}(\alpha^3h^3+4\alpha^2g h^2+6\alpha g^2h+4g^3),\\
\langle m_1 m_t^2\rangle&=\tfrac{\pi\alpha}{2}(\alpha^2h^2v_\parallel+2\alpha g h v_\parallel+\alpha h^2 m_\perp+g^2 v_\parallel+2g h m_\perp).
\end{align*}
\endgroup
Substituting into \eqref{eq:moments}, the entries of $\bA$ (symmetric) in the order $(t,s_1,s_2,\tfrac12t^2,1)$ are
\begingroup\small
\begin{equation}
\begin{aligned}
A_{11}&=g^2+\tfrac{\pi}{2}(\alpha^2h^2+2\alpha g h)+1,\qquad
A_{12}=g\,m_\perp+\tfrac{\pi}{2}(\alpha^2 h v_\parallel+\alpha g v_\parallel+\alpha h m_\perp),\qquad
A_{13}=\tfrac{\pi}{2}\alpha v_\perp\nu,\\
A_{14}&=\tfrac{\pi\alpha h}{4}\bigl(\alpha^2h^2+3\alpha g h+3g^2+3\bigr),\qquad
A_{15}=\tfrac{\pi\alpha h}{2},\qquad
A_{22}=m_\perp^2+\tfrac{\pi}{2}(\alpha^2v_\parallel^2+2\alpha m_\perp v_\parallel)+1,\\
A_{23}&=\tfrac{\pi}{2}\alpha v_\perp(m_\perp+\alpha v_\parallel),\qquad
A_{24}=\tfrac{\pi\alpha}{4}\bigl(\alpha^2h^2v_\parallel+2\alpha g h v_\parallel+\alpha h^2m_\perp+g^2v_\parallel+2g h m_\perp+v_\parallel\bigr),\\
A_{25}&=\tfrac{\pi\alpha v_\parallel}{2},\qquad
A_{33}=1+\tfrac{\pi}{2}\alpha^2v_\perp^2,\qquad
A_{34}=\tfrac{\pi}{4}\alpha v_\perp(\nu^2+1),\qquad
A_{35}=\tfrac{\pi\alpha v_\perp}{2},\qquad
A_{55}=1,\\
A_{44}&=\tfrac34+\tfrac18\Bigl(\pi\bigl(\alpha^4h^4+4\alpha^3g h^3+6\alpha^2g^2h^2+6\alpha^2h^2+4\alpha g^3h+12\alpha g h\bigr)+2g^4+12g^2\Bigr),\\
A_{45}&=\tfrac14\bigl(\pi(\alpha^2h^2+2\alpha g h)+2g^2\bigr)+\tfrac12 ,
\end{aligned}
\label{eq:Aentries}
\end{equation}
\endgroup
and, using $y_+=+1$, $y_-=y_p=-1$,
\begin{equation}
\bb=\Bigl(
g(1-\pi)-\tfrac{\pi\alpha h}{2},\ \
m_\perp(1-\pi)-\tfrac{\pi\alpha v_\parallel}{2},\ \
-\tfrac{\pi\alpha v_\perp}{2},\ \
-\tfrac{\pi}{4}\bigl(\alpha^2h^2+2\alpha g h+2g^2+2\bigr),\ \
-\pi
\Bigr)^\top .
\label{eq:bentries}
\end{equation}
For instance $b_1=\tfrac{1-\pi}{2}g+\tfrac12 g-\tfrac\pi2\nu=g(1-\pi)-\tfrac{\pi\alpha h}{2}$ and $b_4=\tfrac12\bigl[\tfrac{1-\pi}{2}(g^2+1)-\tfrac12(g^2+1)-\tfrac\pi2(\nu^2+1)\bigr]=-\tfrac\pi4(\nu^2+g^2+2)$. In the in-plane case $v_\perp=0$ the third row and column reduce to $A_{33}=1$ and zeros, $b_3=0$, so $w_2=0$ and the system is the $4\times4$ one in $(w_t,w_1,\beta,c)$. All entries were also checked symbolically against a direct summation of the per-cluster moments, and numerically against Monte-Carlo estimates of $\Ex[\bphi\bphi^\top]$ and $\Ex[\bphi y]$.

\paragraph{Existence and uniqueness.}
$\bA=\Ex[\bphi\bphi^\top]\succ0$: if $\bc^\top\bA\bc=\Ex[(\bc^\top\bphi)^2]=0$ then the polynomial $\bc^\top\bphi(t,s_1,s_2)=c_1t+c_2s_1+c_3s_2+\tfrac{c_4}{2}t^2+c_5$ vanishes almost everywhere on $\R^3$ (the mixture has a positive density), hence $\bc=\bzero$. Therefore $\mathcal L$ is strictly convex and $\btheta^\star=\bA^{-1}\bb$ is its unique minimiser.

\subsection{The parabolic-boundary formula and the proof of Theorem~\ref{thm:main}}
\label{app:master}

Fix a test cluster with mean $\bme=(m_t,m_1,m_2)$ and write the test point as $t=m_t+z$, $s_1=m_1+z_1$, $s_2=m_2+z_2$ with $z,z_1,z_2$ i.i.d.\ $\mathcal N(0,1)$. Substituting in Lemma~\ref{lem:reduce},
\begin{equation}
f=\underbrace{w_tm_t+w_1m_1+w_2m_2+c+\tfrac\beta2 m_t^2}_{A_0(\bme)}
+\underbrace{(w_t+\beta m_t)}_{b(\bme)}\,z+\tfrac\beta2 z^2
+\underbrace{w_1z_1+w_2z_2}_{\sim\,\mathcal N(0,\sigma^2)} .
\label{eq:fdecomp}
\end{equation}
The last term is a centred Gaussian of variance $\sigma^2=w_1^2+w_2^2$, independent of $z$. Conditioning on $z$,
\[
\Prob(f>0\mid z)=\Prob\bigl(w_1z_1+w_2z_2>-(A_0+bz+\tfrac\beta2z^2)\bigr)=\Phi\Bigl(\frac{A_0(\bme)+b(\bme)z+\tfrac\beta2z^2}{\sigma}\Bigr),
\]
and averaging over $z$ gives the parabolic-boundary formula
\begin{equation}
  \Prob\bigl(f>0\ \big|\ \bx\in\text{cluster }\bme\bigr)
  = \Ex_z\!\left[\Phi\!\left(\frac{A_0(\bme) + b(\bme)\,z + \tfrac{\beta}{2}z^2}{\sigma}\right)\right],
  \label{eq:boundary}
\end{equation}
with the convention $\Phi(x/0)=\mathbf 1\{x>0\}$ when $\sigma=0$; $\Prob(f<0\mid z)$ is obtained by flipping the sign of the argument. Inserting the three cluster means \eqref{eq:frame} into \eqref{eq:boundary}, with the event $f>0$ for $\mathcal C_+$ and $f<0$ for $\mathcal C_-$ and $\mathcal C_p$, gives \eqref{eq:thm}. The trained weights are those of Appendix~\ref{app:moments}. \hfill$\square$

Geometrically, $\{f=0\}$ is the parabola $s\mapsto-(A_0+bz+\tfrac\beta2z^2)/\sigma$ in the $(z,s)$ plane of feature noise and orthogonal noise, and \eqref{eq:boundary} is the Gaussian mass on one side of it. When $\beta=0$ the parabola is a line and $\Ex_z\Phi((A_0+w_tz)/\sigma)=\Phi\bigl(A_0/\sqrt{w_t^2+\sigma^2}\bigr)$, the linear-classifier formula.

\subsection{Proof of Corollary~\ref{cor:general}(i)}
At $\pi=0$ all poison-induced entries vanish: by \eqref{eq:Aentries}--\eqref{eq:bentries}, $\bA$ is block-diagonal with linear block $\bigl(\begin{smallmatrix}g^2+1&gm_\perp&0\\ gm_\perp&m_\perp^2+1&0\\0&0&1\end{smallmatrix}\bigr)$ and quadratic block $\bQ_0=\bigl(\begin{smallmatrix}\frac14(g^4+6g^2+3)&\frac12(g^2+1)\\ \frac12(g^2+1)&1\end{smallmatrix}\bigr)$, and $\bb=(g,m_\perp,0,0,0)^\top$. The linear block times $(g,m_\perp,0)/(1+S^2)$ equals $(g,m_\perp,0)$ because $g^2+m_\perp^2=S^2$, and the quadratic block has zero right-hand side, so $\btheta_0=\bigl(\tfrac{g}{1+S^2},\tfrac{m_\perp}{1+S^2},0,0,0\bigr)$, i.e.\ $\tw=\tmu/(1+S^2)$, $\beta=c=0$, for every $\bu$. Then $f=\tmu^\top\tx/(1+S^2)$, and on a cluster with mean $\bme$ it is Gaussian with mean $\tmu^\top\bme/(1+S^2)$ and standard deviation $S/(1+S^2)$. For $\bme=\pm\tmu$ this gives $\Phi(S)$ for both clean classes, and for $\bme=\tmu+\alpha\tv$ the event $f<0$ has probability $\Phi\bigl(-(S^2+\alpha ST\rho)/S\bigr)=\Phi(-S-\alpha T\rho)$. \hfill$\square$

\subsection{Proof of Corollary~\ref{cor:general}(ii) (stealth to first order)}
\label{app:stealth}

Since $\bA(\pi)$ and $\bb(\pi)$ are polynomial in $\pi$ and $\bA(0)$ is invertible, $\btheta^\star=\btheta_0+\pi\btheta_1+O(\pi^2)$ with $\btheta_0$ as above. Let $\hat{\bmu}=\tmu/S$ and $\sigma_0=S/(1+S^2)$, and decompose the whitened test noise as $\tz=\zeta\hat{\bmu}+\tz_\perp$ with $\zeta=\hat{\bmu}^\top\tz\sim\mathcal N(0,1)$ independent of $\tz_\perp\perp\hat{\bmu}$. On a clean source point $\tx=\tmu+\tz$,
\[
f=\sigma_0(S+\zeta)+\pi f_1(\zeta,\tz_\perp)+O(\pi^2),\qquad
f_1=\tw_1^\top(\tmu+\tz)+\tfrac{\beta_1}{2}\bigl(g+\tu^\top\tz\bigr)^2+c_1 ,
\]
where $(\tw_1,\beta_1,c_1)$ are the first-order corrections. For fixed $\tz_\perp$ the map $\zeta\mapsto S+\zeta+\tfrac{\pi}{\sigma_0}f_1$ has, for small $\pi$, a unique root $\zeta^\star=-S-\tfrac{\pi}{\sigma_0}f_1(-S,\tz_\perp)+O(\pi^2)$, so $\Prob(f>0\mid\tz_\perp)=\Phi(-\zeta^\star)=\Phi(S)+\tfrac{\pi}{\sigma_0}\varphi(S)f_1(-S,\tz_\perp)+O(\pi^2)$. At $\zeta=-S$ the point is $\tx=\tmu-S\hat{\bmu}+\tz_\perp=\tz_\perp$, so $f_1(-S,\tz_\perp)=\tw_1^\top\tz_\perp+\tfrac{\beta_1}{2}\bigl(g-g+\tu^\top\tz_\perp\bigr)^2+c_1$, using $\tu^\top\hat{\bmu}=g/S$. Averaging over $\tz_\perp$ (which has $\Ex[\tz_\perp]=\bzero$ and $\Ex[(\tu^\top\tz_\perp)^2]=1-g^2/S^2$),
\[
\Prob(f>0\mid\mathcal C_+)=\Phi(S)+\frac{\pi\varphi(S)}{\sigma_0}\,\delta+O(\pi^2),
\qquad
\delta:=\tfrac{\beta_1}{2}\Bigl(1-\tfrac{g^2}{S^2}\Bigr)+c_1 .
\]
On a clean target point $\tx=-\tmu+\tz$, $f=-\sigma_0(S-\zeta)+\pi f_1^-+O(\pi^2)$ with $f_1^-=\tw_1^\top(-\tmu+\tz)+\tfrac{\beta_1}{2}(-g+\tu^\top\tz)^2+c_1$; the root is now $\zeta^\star=S-\tfrac\pi{\sigma_0}f_1^-(S,\tz_\perp)+O(\pi^2)$ and $\Prob(f<0\mid\tz_\perp)=\Phi(\zeta^\star)$. At $\zeta=S$ the point is again $\tx=\tz_\perp$ and the same computation gives $f_1^-(S,\tz_\perp)$ with expectation $\delta$. Hence
\[
\Prob(f<0\mid\mathcal C_-)=\Phi(S)-\frac{\pi\varphi(S)}{\sigma_0}\,\delta+O(\pi^2) .
\]
The two $O(\pi)$ shifts are equal and opposite, and their average $\CACC=\Phi(S)+O(\pi^2)$. The interchange of the $\pi$-expansion with the expectation over $\tz_\perp$ is justified because $f_1$ is a polynomial in $\tz_\perp$ with Gaussian moments. \hfill$\square$

\section{Proofs for the linear regime (Section~\ref{sec:linear})}
\label{app:linear}

\subsection{Proof of Proposition~\ref{prop:linear}}
\emph{(i)} Let $h=0$. By \eqref{eq:frame} the feature means of the three clusters are $g,-g,g$, so the quadratic feature $\phi_4=\tfrac12t^2$ with $t\sim\mathcal N(\pm g,1)$ has the same distribution in every cluster; in particular its mean $m_4=\tfrac12(g^2+1)$ is cluster-independent, and within each cluster it is independent of $(s_1,s_2)$. Hence, in the mixture, $\Cov(\phi_4,y)=\Ex[\phi_4y]-m_4\Ex[y]=m_4\Ex[y]-m_4\Ex[y]=0$, $\Cov(\phi_4,s_j)=\sum_c\omega_c\,m_4\,m_j^c-m_4\sum_c\omega_c m_j^c=0$, and
$\Cov(\phi_4,t)=\Ex[\tfrac12t^3]-m_4\Ex[t]=\sum_c\omega_c\bigl(m_4 m_t^c+m_t^c\bigr)-m_4\sum_c\omega_c m_t^c=\langle m_t\rangle=\tfrac{\pi\alpha h}{2}=0$, using $\tfrac12\Ex[t^3]=\tfrac12(m_t^3+3m_t)=m_4m_t+m_t$. Now let $(\bw',c')$ be the least-squares solution of the model without $\phi_4$ (the linear classifier) and $r=y-\bw'^\top(t,s_1,s_2)-c'$ its residual. The normal equations of that model give $\Ex[r]=0$ and $\Ex[r\,t]=\Ex[r\,s_j]=0$, hence $\Ex[r\phi_4]=\Cov(r,\phi_4)+\Ex[r]\Ex[\phi_4]=\Cov(y,\phi_4)-\sum_j w'_j\Cov(\phi_j,\phi_4)=0$. Thus $(\bw',0,c')$ satisfies all five normal equations \eqref{eq:normal}; by uniqueness it is $\btheta^\star$, so $\beta=0$ and $(\bw,c)=(\bw',c')$. (One can also read this off \eqref{eq:Aentries}--\eqref{eq:bentries}: at $h=0$ the fourth row of $\bA$ is $\tfrac12(g^2+1)$ times its fifth row except for the diagonal entry, and $b_4=\tfrac12(g^2+1)b_5$.)

\emph{(ii)} If $\tu$ is uniform on the unit sphere of $\R^d$ then, by rotational symmetry, $\Ex[\tu\tu^\top]=\mI/d$, so $\Ex[h^2]=\Ex[(\tu^\top\tv)^2]=\|\tv\|^2/d=T^2/d$ and likewise $\Ex[g^2]=S^2/d$. By Proposition~\ref{prop:beta}, $\beta=O(\pi h)$, so the curvature is $O(\pi d^{-1/2})$ and the model is the linear classifier up to that order. \hfill$\square$

\begin{corollary}[Exact weights and metrics in the linear regime]
\label{cor:linear}
In the linear regime the trained parameters are, in whitened coordinates,
\begin{equation}
  \begin{aligned}
  \tw &= (1-\pi)\,(\mI+\mC)^{-1}\Bigl(\tmu - \tfrac{\pi\alpha}{2}\tv\Bigr),
  \qquad
  c = -\pi - \tfrac{\pi\alpha}{2}\,\tv^\top\tw,\\
  \mC &= \tmu\tmu^\top + \tfrac{\pi\alpha}{2}\bigl(\tmu\tv^\top+\tv\tmu^\top\bigr) + \tfrac{\pi\alpha^2}{2}\bigl(1-\tfrac{\pi}{2}\bigr)\tv\tv^\top,
  \end{aligned}
  \label{eq:linw}
\end{equation}
where $\mC$ is the covariance of the cluster means, and
\begin{equation}
  \CACC = \tfrac12\left[\Phi\!\left(\frac{\tw^\top\tmu + c}{\|\tw\|}\right)+\Phi\!\left(\frac{\tw^\top\tmu - c}{\|\tw\|}\right)\right],
  \qquad
  \ASR = \Phi\!\left(-\frac{\tw^\top(\tmu+\alpha\tv) + c}{\|\tw\|}\right).
  \label{eq:linmetrics}
\end{equation}
For a hidden trigger ($\rho=0$), in the basis $(\tmu/S,\tv/T)$ and with $A=\alpha T$,
\begin{equation}
  \begin{aligned}
  w_\mu &= \frac{(1-\pi)\,S\,(1+\tfrac{\pi}{2}A^2)}{D},\qquad
  w_v = -\frac{(1-\pi)\,\tfrac{\pi}{2}A\,(1+2S^2)}{D},\qquad
  c = -\pi-\tfrac{\pi}{2}A\,w_v,\\
  D &= (1+S^2)\Bigl(1+\tfrac{\pi}{2}A^2\bigl(1-\tfrac{\pi}{2}\bigr)\Bigr)-\tfrac{\pi^2}{4}A^2S^2 .
  \end{aligned}
  \label{eq:linw0}
\end{equation}
\end{corollary}

The poisoned cluster enters the normal equations only through its mean $\alpha\tv$, with weight $\pi$: the read-out tilts toward $-\tv$ by $w_v\propto\pi\alpha$, and a triggered test point, displaced by the same $\alpha\tv$, is pushed across the boundary by $\alpha w_v\propto\pi\alpha^2$.

\subsection{Proof of Corollary~\ref{cor:linear}}
The linear model $f=\tw^\top\tx+c$ has normal equations $\Ex[\tx(\tw^\top\tx+c-y)]=\bzero$ and $\Ex[\tw^\top\tx+c-y]=0$. Let $\bar{\bme}=\Ex[\tx]=\sum_c\omega_c\bme_c=\tfrac{\pi\alpha}{2}\tv$ (the $\tmu$ terms cancel: $\tfrac{1-\pi}{2}-\tfrac12+\tfrac\pi2=0$), $\Ex[y]=\tfrac{1-\pi}{2}-\tfrac12-\tfrac\pi2=-\pi$, $\Ex[\tx y]=\tfrac{1-\pi}{2}\tmu+\tfrac12\tmu-\tfrac\pi2(\tmu+\alpha\tv)=(1-\pi)\tmu-\tfrac{\pi\alpha}{2}\tv$, and $\Ex[\tx\tx^\top]=\mI+\bM$ with $\bM=\sum_c\omega_c\bme_c\bme_c^\top=\tmu\tmu^\top+\tfrac{\pi\alpha}{2}(\tmu\tv^\top+\tv\tmu^\top)+\tfrac{\pi\alpha^2}{2}\tv\tv^\top$. The second normal equation gives $c=-\pi-\bar{\bme}^\top\tw$; substituting in the first, $(\mI+\bM-\bar{\bme}\bar{\bme}^\top)\tw=\Ex[\tx y]+\pi\bar{\bme}=(1-\pi)\tmu-\tfrac{\pi\alpha}{2}(1-\pi)\tv$, and $\bM-\bar{\bme}\bar{\bme}^\top=\mC$ as stated, since $\bar{\bme}\bar{\bme}^\top=\tfrac{\pi^2\alpha^2}{4}\tv\tv^\top$. This is \eqref{eq:linw}. On a cluster with mean $\bme$, $f\sim\mathcal N(\tw^\top\bme+c,\|\tw\|^2)$, which gives \eqref{eq:linmetrics}. For $\rho=0$, in the orthonormal basis $(\tmu/S,\tv/T)$ and with $A=\alpha T$, $\mC=\bigl(\begin{smallmatrix}S^2&\pi AS/2\\ \pi AS/2&\pi A^2(1-\pi/2)/2\end{smallmatrix}\bigr)$, $\tmu-\tfrac{\pi\alpha}{2}\tv=(S,-\pi A/2)$ and $\bar{\bme}=(0,\pi A/2)$; inverting the $2\times2$ matrix $\mI+\mC$, whose determinant is $D$, gives \eqref{eq:linw0}: the numerator of $w_\mu$ is $S\bigl(1+\tfrac{\pi A^2}{2}(1-\tfrac\pi2)\bigr)+\tfrac{\pi A}{2}\cdot\tfrac{\pi AS}{2}=S(1+\tfrac{\pi}{2}A^2)$ and that of $w_v$ is $-\tfrac{\pi A}{2}(1+S^2)-\tfrac{\pi AS}{2}S=-\tfrac{\pi A}{2}(1+2S^2)$. \hfill$\square$

\subsection{Proof of Theorem~\ref{thm:linear}}
\label{app:linthm}
Throughout $\rho=0$ and we use \eqref{eq:linw0}. The triggered mean margin is $M_p=w_\mu S+w_vA+c$ and the clean margins are $M_\pm=\pm w_\mu S+c$; $\ASR=\Phi(-M_p/\|\bw\|)$ and $\CACC=\tfrac12[\Phi(M_+/\|\bw\|)+\Phi(-M_-/\|\bw\|)]$ with $\|\bw\|^2=w_\mu^2+w_v^2$.

\emph{(a)} Since $D=(1+S^2)(1+\tfrac\pi2A^2)+O(\pi^2)$, we get $w_\mu=\tfrac{(1-\pi)S}{1+S^2}+O(\pi^2)$, $w_v=-\tfrac{\pi A(1+2S^2)}{2(1+S^2)}+O(\pi^2)$, $c=-\pi+O(\pi^2)$ and $\|\bw\|=w_\mu+O(\pi^2)$. Hence
\[
\frac{M_p}{\|\bw\|}=S+\frac{w_vA+c}{w_\mu}+O(\pi^2)=S-\pi\Bigl[\frac{(1+2S^2)A^2}{2S}+\frac{1+S^2}{S}\Bigr]+O(\pi^2),
\]
and $\Phi(-M_p/\|\bw\|)=\Phi(-S)+\pi\varphi(S)[\cdots]+O(\pi^2)$. For the clean classes, $M_\pm/\|\bw\|=\pm S+c/w_\mu+O(\pi^2)$, and $\Phi(S+c/w_\mu)+\Phi(S-c/w_\mu)=2\Phi(S)+O(\pi^2)$.

\emph{(b)} Put $\pi=e/A^2$ and let $A\to\infty$ (so $\pi\to0$). Then $\pi A^2=e$, $\pi A=e/A\to0$, $\pi^2A^2=\pi e\to0$, and $D\to(1+S^2)(1+\tfrac e2)$. Therefore $w_\mu\to S/(1+S^2)$, $w_v=O(\pi A)\to0$, $w_vA\to-\tfrac{e(1+2S^2)}{2(1+S^2)(1+e/2)}=-\tfrac{(1+2S^2)}{(1+S^2)}\tfrac{e}{e+2}$, $c=-\pi-\tfrac{\pi A}{2}w_v\to0$, $\|\bw\|\to S/(1+S^2)$. Hence $M_p/\|\bw\|\to S-\tfrac{1+2S^2}{S}\tfrac{e}{e+2}$ and $M_\pm/\|\bw\|\to\pm S$, which is \eqref{eq:linlaw} and $\CACC\to\Phi(S)$.

\emph{(c)} $e\mapsto e/(e+2)$ increases from $0$ to $1$, so $\ASR_{\rm lin}$ increases from $\Phi(-S)$ to $\Phi\bigl(-S+\tfrac{1+2S^2}{S}\bigr)=\Phi\bigl(\tfrac{1+S^2}{S}\bigr)$.

\emph{(d)} Setting $\ASR_{\rm lin}(e)=a$ gives $\tfrac{1+2S^2}{S}\tfrac{e}{e+2}=S+\Phi^{-1}(a)$, i.e.\ $\tfrac{e}{e+2}=\kappa_a$, which has a solution $e_a=\tfrac{2\kappa_a}{1-\kappa_a}\ge0$ iff $0\le\kappa_a<1$, i.e.\ iff $\Phi(-S)\le a<\Phi\bigl(\tfrac{1+S^2}{S}\bigr)$. Then $\alpha=\sqrt{e_a/\pi}/T$. \hfill$\square$

\section{Proofs for the feature-learning regime (Section~\ref{sec:feature})}
\label{app:fl}

\subsection{Proof of Proposition~\ref{prop:beta} (first-order expansion)}
\label{app:beta}
Write $\bA=\bA_0+\pi\bA_1$ and $\bb=\bb_0+\pi\bb_1$ with $\bA_0,\bb_0$ the $\pi=0$ parts of \eqref{eq:Aentries}--\eqref{eq:bentries}; all entries are affine in $\pi$ so this is exact. Then $\btheta^\star=\btheta_0+\pi\btheta_1+O(\pi^2)$ with $\btheta_0=\bA_0^{-1}\bb_0$ as in the proof of Corollary~\ref{cor:general}(i) and
\begin{equation}
\btheta_1=\bA_0^{-1}\bigl(\bb_1-\bA_1\btheta_0\bigr)=:\bA_0^{-1}\br .
\label{eq:theta1def}
\end{equation}
Because $\bA_0$ is block-diagonal between the linear coordinates and the pair $(\tfrac12t^2,1)$, the curvature is $\beta_1=[\bQ_0^{-1}(r_4,r_5)^\top]_1$ with $\bQ_0$ the quadratic block; since $\det\bQ_0=\tfrac14(g^4+6g^2+3)-\tfrac14(g^2+1)^2=\tfrac{2g^2+1}{2}$,
\begin{equation}
\beta_1=\frac{2}{2g^2+1}\Bigl(r_4-\frac{g^2+1}{2}\,r_5\Bigr).
\label{eq:beta1r}
\end{equation}
It remains to compute $r_4$ and $r_5$. Since $\btheta_0$ has only the two components $(g,m_\perp)/(1+S^2)$, $(\bA_1\btheta_0)_k=\bigl(A^{(1)}_{k1}g+A^{(1)}_{k2}m_\perp\bigr)/(1+S^2)$ where $A^{(1)}_{kj}$ is the coefficient of $\pi$ in $A_{kj}$. From \eqref{eq:Aentries}, $A^{(1)}_{51}=\tfrac{\alpha h}{2}$, $A^{(1)}_{52}=\tfrac{\alpha v_\parallel}{2}$, so using $gh+m_\perp v_\parallel=ST\rho$,
\[
r_5=-1-\frac{\alpha(gh+m_\perp v_\parallel)}{2(1+S^2)}=-1-\frac{\alpha ST\rho}{2(1+S^2)} .
\]
Similarly $A^{(1)}_{41}=\tfrac{\alpha h}{4}(\alpha^2h^2+3\alpha gh+3g^2+3)$ and $A^{(1)}_{42}=\tfrac{\alpha}{4}(\alpha^2h^2v_\parallel+2\alpha ghv_\parallel+\alpha h^2m_\perp+g^2v_\parallel+2ghm_\perp+v_\parallel)$; multiplying by $g$ and $m_\perp$ respectively, adding, and replacing $m_\perp v_\parallel=ST\rho-gh$ and $m_\perp^2=S^2-g^2$, all terms in $g^2h^2$, $g^3h$ and $gh$ cancel and one finds
\[
A^{(1)}_{41}g+A^{(1)}_{42}m_\perp=\frac{\alpha}{4}\Bigl[ST\rho\bigl(\alpha^2h^2+2\alpha gh+g^2+1\bigr)+\alpha h^2S^2+2gh(S^2+1)\Bigr],
\]
so that, with $b^{(1)}_4=-\tfrac14(\alpha^2h^2+2\alpha gh+2g^2+2)$,
\[
r_4=-\frac{\alpha^2h^2+2\alpha gh+2g^2+2}{4}-\frac{\alpha\bigl[ST\rho(\alpha^2h^2+2\alpha gh+g^2+1)+\alpha h^2S^2+2gh(S^2+1)\bigr]}{4(1+S^2)} .
\]
Substituting in \eqref{eq:beta1r}: the constant $-\tfrac{2g^2+2}{4}$ in $r_4$ cancels against $-\tfrac{g^2+1}{2}\cdot(-1)$ from $r_5$, the terms $\alpha ST\rho(g^2+1)$ cancel between $r_4$ and $-\tfrac{g^2+1}{2}r_5$, and collecting the rest over the denominator $4(1+S^2)$,
\[
r_4-\tfrac{g^2+1}{2}r_5=-\frac{\alpha h\bigl[ST\rho\,h\,\alpha^2+\bigl((1+2S^2)h+2ST\rho\,g\bigr)\alpha+4g(1+S^2)\bigr]}{4(1+S^2)} ,
\]
which gives \eqref{eq:beta}. The remaining components of $\btheta_1$ are finite, so $\tw=\tmu/(1+S^2)+O(\pi)$ and $c=O(\pi)$. In the feature-learning regime with $\rho=0$ one has $g=0$, $h=T$, and \eqref{eq:beta} becomes $\beta=-\pi\tfrac{(1+2S^2)}{2(1+S^2)}\alpha^2T^2+O(\pi^2)$. (Expression \eqref{eq:beta} was also verified symbolically against $\bA^{-1}\bb$ at random rational points.) \hfill$\square$

\begin{corollary}[Exact weights in the feature-learning regime]
\label{cor:flweights}
In the feature-learning regime with a hidden trigger ($\bu=\bu^\star$, $\rho=0$), in the frame $(\tv/T,\tmu/S)$ of feature and residual signal, with $A=\alpha T$,
\begin{equation}
  \begin{aligned}
  w_\mu &= \frac{2S(1-\pi)\bigl[4+\pi A^2(A^2+6)-2\pi^2A^2\bigr]}{\Delta},
  \qquad
  \beta = -\frac{4\pi A^2(1+2S^2)(1-\pi)^2}{\Delta},\\[2pt]
  w_v &= -\frac{2\pi A(1-\pi)(1+2S^2)(\pi A^2+2)}{\Delta},\\[2pt]
  c &= \frac{\pi\bigl[2A^2(1+2S^2)-8(1+S^2)-\pi A^4-2\pi A^2(8S^2+7)+2\pi^2A^2(4S^2+3)\bigr]}{\Delta},
  \end{aligned}
  \label{eq:flweights}
\end{equation}
where
\begin{equation}
  \Delta=16\det\bA=8(1+S^2)+2\pi A^2(1+S^2)(A^2+6)-\pi^2A^2\bigl[(1+2S^2)A^2+2(3+4S^2)\bigr].
  \label{eq:Delta}
\end{equation}
\end{corollary}

As $A\to\infty$, $\Delta\sim\pi\bigl[2(1+S^2)-\pi(1+2S^2)\bigr]A^4$ while the numerator of $\beta$ grows only as $A^2$, so $\beta=O(A^{-2})$: the detector saturates. When $\pi\to0$ and $A\to\infty$, $\Delta/8(1+S^2)\to1+\tfrac{\pi A^4}{4}$, so $\pi A^4$ is the only combination of $\pi$ and $A$ that stays finite.

\subsection{Proof of Corollary~\ref{cor:flweights} (exact weights, $\rho=0$)}
\label{app:flweights}
In the feature-learning regime $\tu=\tv/T$, so $h=T$, $g=S\rho$, $v_\parallel=v_\perp=0$, $m_\perp=S\sqrt{1-\rho^2}$, and $\alpha$ enters only through $A=\alpha h=\alpha T$. With $\rho=0$ ($g=0$, $m_\perp=S$), the $4\times4$ system of \eqref{eq:Aentries}--\eqref{eq:bentries} in the unknowns $(w_v,w_\mu,\beta,c)$ (feature first) is
\[
\bA=\begin{pmatrix}
1+\tfrac{\pi}{2}A^2 & \tfrac{\pi}{2}AS & \tfrac{\pi}{4}A(A^2+3) & \tfrac{\pi}{2}A\\[2pt]
\cdot & 1+S^2 & \tfrac{\pi}{4}A^2S & 0\\[2pt]
\cdot & \cdot & \tfrac34+\tfrac{\pi}{8}A^2(A^2+6) & \tfrac12+\tfrac{\pi}{4}A^2\\[2pt]
\cdot & \cdot & \cdot & 1
\end{pmatrix},
\qquad
\bb=\begin{pmatrix}-\tfrac{\pi}{2}A\\ S(1-\pi)\\ -\tfrac{\pi}{4}(A^2+2)\\ -\pi\end{pmatrix}.
\]
Solving by Cramer's rule gives \eqref{eq:flweights}; the identity $16\det\bA=\Delta$ and the four quotients were verified by exact symbolic computation (substituting \eqref{eq:flweights} back into $\bA\btheta=\bb$ yields zero identically). The leading behaviour as $A\to\infty$ follows from the highest powers: $\Delta=\pi\bigl[2(1+S^2)-\pi(1+2S^2)\bigr]A^4+O(A^2)$, while the numerators of $\beta$, $w_vA$ and $c$ are $O(A^2)$, $O(A^4)$ and $O(A^4)$. \hfill$\square$

For the ceiling at general $\rho$ we also record the exact weights of the trigger-aligned frame $g=S\rho$, $m_\perp=S\bar\rho$ ($\bar\rho=\sqrt{1-\rho^2}$), $h=T$, $v_\parallel=0$, again in terms of $A=\alpha T$:
\begin{equation}
w_v=\frac{2(1-\pi)}{\Delta}\,P_v,\qquad
w_\mu=\frac{2S\bar\rho(1-\pi)}{\Delta}\,P_\mu,\qquad
\beta=-\frac{4A\pi(1-\pi)}{\Delta}\,P_\beta,\qquad
c=-\frac{\pi}{\Delta}\,P_c,
\label{eq:alignedweights}
\end{equation}
where $\Delta=16\det\bA$ and, grouped by powers of $A$,
\begingroup\small
\begin{align*}
\Delta=\ &8(1+S^2)(1+2S^2\rho^2)
+8S\pi\rho\,(2S^2\rho^2+2S^2+3)\,A
+8S\pi\rho\,(1+S^2-S^2\pi)\,A^3\\
&+\pi\bigl(2(1+S^2)-\pi(1+2S^2)\bigr)\,A^4\\
&+\bigl[\pi(12+12S^2-8S^4\rho^4+16S^4\rho^2+12S^2\rho^2)
+\pi^2(8S^4\rho^4-16S^4\rho^2+8S^2\rho^2-8S^2-6)\bigr]A^2,\\
P_\beta=\ &4S\rho(1+S^2)
+\bigl[(1-\pi)(1+2S^2)+2(1+\pi)S^2\rho^2\bigr]A
+S\rho\,A^2,\\
P_v=\ &4S\rho(2S^2\rho^2+1)
+2\pi\bigl(4S^4\rho^4-4S^4\rho^2+4S^2\rho^2-2S^2-1\bigr)A\\
&+2S\pi\rho\bigl(2S^2\pi\rho^2-2S^2\pi+2S^2\rho^2-\pi+2\bigr)A^2
+\pi\bigl(2S^2\pi\rho^2-2S^2\pi+4S^2\rho^2-\pi\bigr)A^3
+S\pi\rho\,A^4,\\
P_\mu=\ &4(2S^2\rho^2+1)+4S\pi\rho\bigl(2S^2\rho^2+3\bigr)A
+2\pi\bigl(2S^2\pi\rho^2+4S^2\rho^2-\pi+3\bigr)A^2\\
&+2S\pi\rho(\pi+2)A^3+\pi A^4,\\
P_c=\ &8(1+S^2)(1+2S^2\rho^2)\\
&+4S\rho\bigl(2S^4\pi\rho^2-2S^4\rho^2+4S^2\pi\rho^2+6S^2\pi-2S^2+7\pi-1\bigr)A\\
&+2\bigl(2S^4\pi^2\rho^4-2S^4\pi^2\rho^2+4S^4\pi\rho^2-2S^4\rho^4-2S^4\rho^2
+5S^2\pi^2\rho^2-4S^2\pi^2\\
&\qquad\quad+8S^2\pi\rho^2+8S^2\pi-3S^2\rho^2-2S^2-3\pi^2+7\pi-1\bigr)A^2\\
&+2S\rho\bigl(S^2\pi\rho^2-S^2\rho^2+\pi^2+4\pi-1\bigr)A^3+\pi\,A^4 .
\end{align*}
\endgroup
These reduce to \eqref{eq:flweights} at $\rho=0$ and were verified symbolically against the $4\times4$ solve.

\subsection{Proof of Theorem~\ref{thm:fl}}
\label{app:flthm}
In the frame of Corollary~\ref{cor:flweights} the cluster means are $\bme_\pm=(0,\pm S)$ and $\bme_p=(A,S)$, so by Theorem~\ref{thm:main}, with $\sigma=|w_\mu|$,
\[
A_0^\pm=\pm w_\mu S+c,\quad b^\pm=w_v,\qquad
A_0^p=w_\mu S+w_vA+\tfrac\beta2A^2+c,\quad b^p=w_v+\beta A .
\]

\emph{(a) Small poison.} Expanding \eqref{eq:flweights} to first order in $\pi$ at fixed $A$ (equivalently, applying Appendix~\ref{app:beta} with $g=0$, $h=T$, $m_\perp=S$, $v_\parallel=0$),
\begin{equation}
\begin{aligned}
w_\mu&=\frac{S(1-\pi)}{1+S^2}+O(\pi^2),\qquad
w_v=-\frac{\pi A(1+2S^2)}{2(1+S^2)}+O(\pi^2),\\
\beta&=-\frac{\pi A^2(1+2S^2)}{2(1+S^2)}+O(\pi^2),\qquad
c=\frac{\pi\bigl[A^2(1+2S^2)-4(1+S^2)\bigr]}{4(1+S^2)}+O(\pi^2).
\end{aligned}
\label{eq:flfirst}
\end{equation}
Since $w_v,\beta,c=O(\pi)$, the argument of $\Phi$ in the $\ASR$ of \eqref{eq:thm} is $-S-\delta(z)+O(\pi^2)$ with $\delta(z)=\bigl[w_vA+\tfrac\beta2A^2+c+b^pz+\tfrac\beta2z^2\bigr]/w_\mu$, and $\Ex_z\Phi(-S-\delta(z))=\Phi(-S)-\varphi(S)\Ex_z[\delta(z)]+O(\pi^2)$ with $\Ex_z[\delta]=\bigl[w_vA+\tfrac\beta2A^2+\tfrac\beta2+c\bigr]/w_\mu$ (because $\Ex z=0$, $\Ex z^2=1$). Substituting \eqref{eq:flfirst} over the common denominator $4(1+S^2)$,
\[
w_vA+\tfrac\beta2A^2+\tfrac\beta2+c
=-\pi\,\frac{(1+2S^2)\,A^2(A^2+2)+4(1+S^2)}{4(1+S^2)},
\]
and with $1/w_\mu=(1+S^2)/S+O(\pi)$ this is \eqref{eq:asr_expand}. The identification of the extra term relative to Theorem~\ref{thm:linear}(a) is immediate: $\tfrac{(1+2S^2)A^2(A^2+2)}{4S}=\tfrac{(1+2S^2)A^2}{2S}+\tfrac{(1+2S^2)A^4}{4S}$. The clean accuracy statement is Corollary~\ref{cor:general}(ii); explicitly, the two clean shifts are $\mp\varphi(S)\bigl(\tfrac\beta2+c\bigr)/w_\mu$ and $\tfrac\beta2+c=-\pi+O(\pi^2)$ by \eqref{eq:flfirst}.

\emph{(b) Stealth limit.} Put $\pi=\varrho/A^4$ and let $A\to\infty$. Then $\pi A^4=\varrho$, while $\pi A^2=\varrho/A^2\to0$ and $\pi^2A^4=\pi\varrho\to0$. Term by term in \eqref{eq:flweights},
\begin{gather*}
\Delta\to 2(1+S^2)(4+\varrho),\qquad
w_\mu\to\frac{2S(4+\varrho)}{2(1+S^2)(4+\varrho)}=\frac{S}{1+S^2},\\
\frac\beta2A^2=-\frac{2\pi A^4(1+2S^2)(1-\pi)^2}{\Delta}\to-\frac{1+2S^2}{1+S^2}\cdot\frac{\varrho}{4+\varrho}\,,
\end{gather*}

while $w_vA=-2(1+2S^2)(1-\pi)\pi A^2(\pi A^2+2)/\Delta\to0$, $c\to0$ (every term of $\pi P_c$ is $O(\pi\varrho)$ or $O(\pi A^2)$), $\beta\to0$ and $b^p=w_v+\beta A\to0$. Hence in the triggered integral of \eqref{eq:thm} the $z$-dependence disappears and the integral collapses to a single $\Phi$ with argument $-A_0^p/w_\mu\to-S+\tfrac{1+2S^2}{S}\tfrac{\varrho}{4+\varrho}$, which is \eqref{eq:fllaw}. For the clean clusters $A_0^\pm/w_\mu\to\pm S$ and $b^\pm,\beta\to0$, so $\CACC\to\Phi(S)$. (The convergence of the $z$-integrals is dominated: for large $A$ the integrands are uniformly bounded and converge pointwise.)

\emph{(c) Ceiling.} Fix $\pi\in(0,1)$ and let $A\to\infty$. We first treat $\rho=0$ using \eqref{eq:flweights}. Let $D_\infty:=2(1+S^2)-\pi(1+2S^2)=2(1+S^2)(1-\pi)+\pi>0$, so that $\Delta=\pi D_\infty A^4+O(A^2)$. Then
\begin{gather*}
w_\mu\to\frac{2S(1-\pi)}{D_\infty},\qquad
w_vA\to-\frac{2\pi(1-\pi)(1+2S^2)}{D_\infty},\\
\frac\beta2A^2\to-\frac{2(1+2S^2)(1-\pi)^2}{D_\infty},\qquad
c\to-\frac{\pi}{D_\infty},
\end{gather*}
and $b^p=w_v+\beta A=O(A^{-1})\to0$, $\beta\to0$. Summing, the numerator of the limit of $A_0^p$ is
\begin{align*}
&2S^2(1-\pi)-2\pi(1-\pi)(1+2S^2)-2(1+2S^2)(1-\pi)^2-\pi\\
&\qquad=(1-\pi)\bigl[2S^2-2(1+2S^2)\bigr]-\pi
=-2(1+S^2)(1-\pi)-\pi=-D_\infty ,
\end{align*}
so $A_0^p\to-1$. Since $\beta\to0$ and $b^p\to0$, the triggered statistic becomes linear-Gaussian and $\ASR\to\Phi\bigl(-A_0^p/\sqrt{w_\mu^2+(b^p)^2}\bigr)\to\Phi\bigl(1/w_{\mu,\infty}\bigr)=\Phi\bigl(D_\infty/(2S(1-\pi))\bigr)$, which is Theorem~\ref{thm:fl}(c) because $D_\infty/(2S(1-\pi))=\tfrac{1+S^2}{S}+\tfrac{\pi}{2S(1-\pi)}$. For general $\rho$ the same computation uses \eqref{eq:alignedweights}. The leading coefficients give $w_\mu\to2S\bar\rho(1-\pi)/D_\infty$, $w_v\to2S\rho(1-\pi)/D_\infty$, $\beta A\to-4S\rho(1-\pi)/D_\infty$, hence $b^p=w_v+\beta(S\rho+A)\to-2S\rho(1-\pi)/D_\infty$ and $w_\mu^2+(b^p)^2\to4S^2(1-\pi)^2(\bar\rho^2+\rho^2)/D_\infty^2=\bigl(2S(1-\pi)/D_\infty\bigr)^2$, independent of $\rho$. In $A_0^p=w_v(S\rho+A)+w_\mu S\bar\rho+c+\tfrac\beta2(S\rho+A)^2$ the $O(A)$ terms $w_vA+\tfrac12(\beta A)A$ cancel at leading order, and collecting the $O(1)$ terms from the two leading coefficients of each polynomial (the $A^4$ and $A^3$ coefficients of $\Delta$ and $P_v$, the $A^2$ and $A^1$ coefficients of $P_\beta$, and the $A^4$ coefficient of $P_c$), every $\rho$-dependent term cancels and one again finds $A_0^p\to-1$. Hence Theorem~\ref{thm:fl}(c) holds for all $\rho$; we also confirmed it by exact symbolic limits at several rational $(S,\rho,\pi)$ and numerically to six decimals. As $\pi\to0$ the correction vanishes and the ceiling tends to $\Phi\bigl(\tfrac{1+S^2}{S}\bigr)$.

\emph{(d) Budget curve.} $\varrho\mapsto\varrho/(4+\varrho)$ increases from $0$ to $1$, so $\ASR_{\rm fl}(\varrho)=a$ has a solution iff $\kappa_a\in[0,1)$, namely $\tfrac{\varrho}{4+\varrho}=\kappa_a$, i.e.\ $\varrho_a=\tfrac{4\kappa_a}{1-\kappa_a}$, and then $\alpha=(\varrho_a/\pi)^{1/4}/T$. The condition $\kappa_a<1$ is again $a<\Phi\bigl(\tfrac{1+S^2}{S}\bigr)$. \hfill$\square$

\subsection{Learning the feature direction end to end}
\label{app:endtoend}
Definition~\ref{def:fl} fixes $\bu=\bu^\star$. To check that this is the direction a network that also trains $\bu$ would select, note that for any fixed $\bu$ the optimal read-out is $\bA^{-1}\bb$, so end-to-end training of $(\bw,\beta,c,\bu)$ minimises the profile loss $\mathcal L^\star(\bu)=\tfrac12\bigl(1-\bb^\top\bA^{-1}\bb\bigr)$ over $\bu^\top\bSig\bu=1$. In the isotropic setting $S=1.5$, $\rho=0$, $\pi=0.1$ we minimised $\mathcal L^\star$ over in-plane directions and, independently, ran Adam on all of $(\bw,\beta,c,\bu)$ for the population loss: both converge to the same direction, whose overlap $h/T$ with the whitened trigger is $0.95$ at $A=4$ and decreases to $0.64$ at $A=16$. The direction that would maximise the ASR at fixed $A$ is closer to the signal ($h^\star/T$ from $0.81$ to $0.38$ over the same range), but because $\ASR$ is flat near its maximum in $h$ the ASR obtained by the learned direction is within $0.03$ of the maximal one. End-to-end feature learning therefore does not target the attack, yet lands in the feature-learning region $h\gg T\pi^{1/4}$ of Corollary~\ref{cor:channels} and is nearly ASR-optimal as a by-product; $\bu^\star$ is its weak-trigger limit.

\section{Two channels: derivation of Corollary~\ref{cor:channels}}
\label{app:channels}
\begin{figure}[h]
  \centering
  \includegraphics[width=0.95\linewidth]{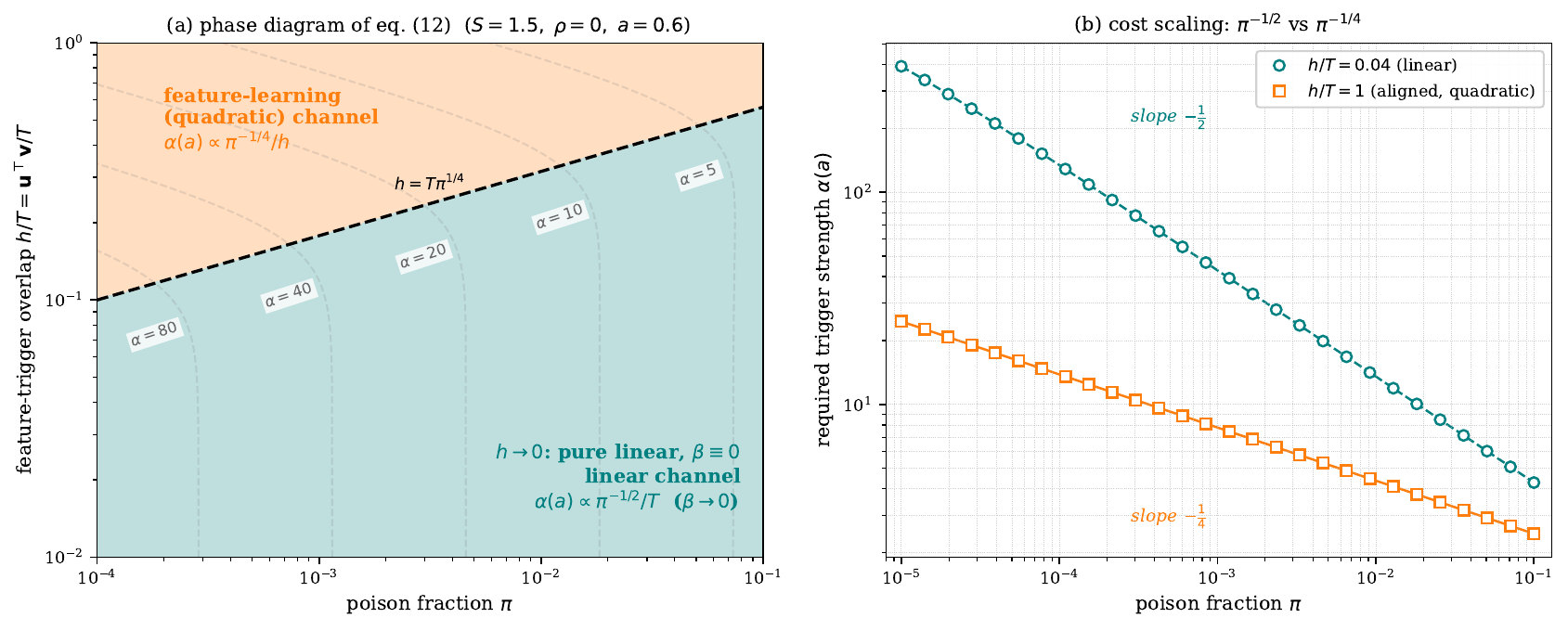}
  \caption{\textbf{Two backdoor channels.} \emph{\textbf{(a)}~The $(\pi,\,h/T)$ plane splits at the crossover $h = T\pi^{1/4}$ (dashed blue) into the feature-learning channel, $\alpha\propto\pi^{-1/4}$, and the linear channel, $\alpha\propto\pi^{-1/2}$ (at $h=0$, $\beta\equiv 0$). Grey contours are lines of constant exact required trigger strength from Theorem~\ref{thm:main}. \textbf{(b)}~Log-log plot of $\alpha(\pi;a)$ for the two extreme alignments, confirming the two exponents. Parameters: $S=1.5$, $\rho=0$, target ASR $a=0.6$.}}
  \label{fig:phase}
\end{figure}

Let $\rho=0$ and let $\bu$ be in-plane with overlap $0\le h\le T$. In the frame of Lemma~\ref{lem:reduce} the trigger is $\tv=h\be_1+v_\parallel\be_2$ with $v_\parallel^2=T^2-h^2$ (the third coordinate is absent in-plane); in this appendix $v_\perp:=|v_\parallel|=\sqrt{T^2-h^2}$ denotes the magnitude of the trigger component orthogonal to the feature, as in Corollary~\ref{cor:channels}. The trigger displaces the poisoned cluster by $\alpha h$ along the feature and by $\alpha v_\perp$ orthogonally to it. By Proposition~\ref{prop:beta}, $\beta\propto\pi h$ and the quadratic detector acts on the feature displacement exactly as in Theorem~\ref{thm:fl} with $A$ replaced by $\alpha h$: its stealth invariant is $\pi(\alpha h)^4$ and it moves the margin argument by $\tfrac{1+2S^2}{S}\tfrac{\varrho'}{4+\varrho'}$, $\varrho'=\pi\alpha^4h^4$. The orthogonal displacement is invisible to the detector and acts through the linear read-out exactly as in Theorem~\ref{thm:linear} with $A$ replaced by $\alpha v_\perp$: its invariant is $e'=\pi\alpha^2v_\perp^2$ and it moves the argument by $\tfrac{1+2S^2}{S}\tfrac{e'}{e'+2}$. Both channels push the argument from $-S$ toward the common ceiling $\tfrac{1+S^2}{S}$ with the same total budget. Writing both saturating fractions as $G/(4+G)$ with $G=2e'$ and $G=\varrho'$ respectively, and adding the two contributions to a single generalised energy $G(\alpha)=2\pi\alpha^2v_\perp^2+\pi\alpha^4h^4$, gives the approximate law $\ASR\simeq\Phi\bigl(-S+\tfrac{1+2S^2}{S}\tfrac{G}{4+G}\bigr)$. Setting it equal to $a$ gives $G=\varrho_a=\tfrac{4\kappa_a}{1-\kappa_a}$, a quadratic in $\alpha^2$, $\pi h^4\alpha^4+2\pi v_\perp^2\alpha^2-\varrho_a=0$, whose positive root is
\begin{equation}
  \alpha(\pi;a,h)
  = \left(\frac{\sqrt{v_\perp^4 + \varrho_a\,h^4/\pi} - v_\perp^2}{h^4}\right)^{1/2},
  \qquad
  v_\perp^2=T^2-h^2,\qquad
  \varrho_a = \frac{4\kappa_a}{1-\kappa_a}.
  \label{eq:unified}
\end{equation} At $h=0$ the equation reduces to $2\pi\alpha^2T^2=\varrho_a$, i.e.\ $e=\varrho_a/2=e_a$, which is \eqref{eq:lincost}; at $h=T$ it reduces to $\pi\alpha^4T^4=\varrho_a$, which is \eqref{eq:cost}. The two terms of $G$ balance at $\alpha^2=2v_\perp^2/h^4$; inserting the linear-channel strength $\alpha^2\sim e_a/(\pi T^2)$ gives the crossover $h^4\sim2\pi T^2v_\perp^2/e_a\sim\pi T^4$, i.e.\ $h\sim T\pi^{1/4}$. The additive combination of the two saturating fractions is exact at both ends and, in the interior of the two-channel region, agrees with the exact root of $\ASR(\alpha)=a$ computed from Theorem~\ref{thm:main} to within $10$--$30\%$; the phase diagram of Figure~\ref{fig:phase} uses the exact root for the contours and the two extreme alignments for the slopes.

\section{Additional experimental results}
\label{app:exp}
Table~\ref{tab:cacc_asr_results} reports the clean accuracy and the attack success rate of the trained quadratic neuron (feature-learning regime) and of the trained linear model (linear regime) on the four datasets of Section~\ref{sec:experiments}, for poison fractions $\pi\in\{0.01,0.02,0.05\}$ and trigger strengths $\alpha\in\{1,3,5\}$, with the protocol of Section~\ref{sec:experiments}. In every case the clean accuracy is unchanged by the poison, while the attack success rate is uniformly larger in the feature-learning regime.

\begin{table}[h]
\centering
\caption{Clean Accuracy (CACC) and Attack Success Rate (ASR) for different values of poison fraction ($\pi$) and trigger strength ($\alpha$).}
\label{tab:cacc_asr_results}
\resizebox{\textwidth}{!}{%
\begin{tabular}{ll cccccc c cccccc}
\toprule
& & \multicolumn{6}{c}{\textbf{Rich Regime}} & & \multicolumn{6}{c}{\textbf{Lazy Regime}} \\
\cmidrule{3-8} \cmidrule{10-15}
\textbf{Dataset} & \textbf{$\pi$} & \multicolumn{2}{c}{$\alpha=1.0$} & \multicolumn{2}{c}{$\alpha=3.0$} & \multicolumn{2}{c}{$\alpha=5.0$} & & \multicolumn{2}{c}{$\alpha=1.0$} & \multicolumn{2}{c}{$\alpha=3.0$} & \multicolumn{2}{c}{$\alpha=5.0$} \\
& & CACC & ASR & CACC & ASR & CACC & ASR & & CACC & ASR & CACC & ASR & CACC & ASR \\
\midrule
 & 0.01 & \cellcolor{myteal!93!white} 0.93 & \cellcolor{myorange!7!white} 0.08 & \cellcolor{myteal!93!white} 0.93 & \cellcolor{myorange!30!white} 0.30 & \cellcolor{myteal!93!white} 0.93 & \cellcolor{myorange!70!white} 0.70 & & \cellcolor{myteal!93!white} 0.93 & \cellcolor{myorange!8!white} 0.08 & \cellcolor{myteal!93!white} 0.93 & \cellcolor{myorange!14!white} 0.14 & \cellcolor{myteal!93!white} 0.93 & \cellcolor{myorange!32!white} 0.33 \\
\multirow{3}{*}{\textbf{Synthetic}} & 0.02 & \cellcolor{myteal!93!white} 0.93 & \cellcolor{myorange!8!white} 0.08 & \cellcolor{myteal!93!white} 0.93 & \cellcolor{myorange!47!white} 0.47 & \cellcolor{myteal!93!white} 0.93 & \cellcolor{myorange!83!white} 0.84 & & \cellcolor{myteal!93!white} 0.93 & \cellcolor{myorange!8!white} 0.09 & \cellcolor{myteal!92!white} 0.93 & \cellcolor{myorange!21!white} 0.21 & \cellcolor{myteal!92!white} 0.93 & \cellcolor{myorange!51!white} 0.52 \\
 & 0.05 & \cellcolor{myteal!93!white} 0.93 & \cellcolor{myorange!10!white} 0.10 & \cellcolor{myteal!92!white} 0.93 & \cellcolor{myorange!67!white} 0.67 & \cellcolor{myteal!93!white} 0.93 & \cellcolor{myorange!92!white} 0.92 & & \cellcolor{myteal!92!white} 0.93 & \cellcolor{myorange!11!white} 0.12 & \cellcolor{myteal!92!white} 0.92 & \cellcolor{myorange!42!white} 0.43 & \cellcolor{myteal!92!white} 0.93 & \cellcolor{myorange!74!white} 0.74 \\
\midrule
 & 0.01 & \cellcolor{myteal!97!white} 0.97 & \cellcolor{myorange!97!white} 0.97 & \cellcolor{myteal!97!white} 0.97 & \cellcolor{myorange!99!white} 0.99 & \cellcolor{myteal!97!white} 0.97 & \cellcolor{myorange!99!white} 1.00 & & \cellcolor{myteal!85!white} 0.85 & \cellcolor{myorange!14!white} 0.15 & \cellcolor{myteal!85!white} 0.85 & \cellcolor{myorange!27!white} 0.28 & \cellcolor{myteal!85!white} 0.85 & \cellcolor{myorange!59!white} 0.59 \\
\multirow{3}{*}{\textbf{MNIST}} & 0.02 & \cellcolor{myteal!97!white} 0.97 & \cellcolor{myorange!98!white} 0.99 & \cellcolor{myteal!97!white} 0.97 & \cellcolor{myorange!99!white} 1.00 & \cellcolor{myteal!97!white} 0.97 & \cellcolor{myorange!99!white} 1.00 & & \cellcolor{myteal!85!white} 0.85 & \cellcolor{myorange!16!white} 0.16 & \cellcolor{myteal!85!white} 0.85 & \cellcolor{myorange!44!white} 0.44 & \cellcolor{myteal!85!white} 0.85 & \cellcolor{myorange!85!white} 0.86 \\
 & 0.05 & \cellcolor{myteal!97!white} 0.97 & \cellcolor{myorange!99!white} 0.99 & \cellcolor{myteal!97!white} 0.97 & \cellcolor{myorange!99!white} 1.00 & \cellcolor{myteal!97!white} 0.97 & \cellcolor{myorange!99!white} 1.00 & & \cellcolor{myteal!85!white} 0.85 & \cellcolor{myorange!22!white} 0.22 & \cellcolor{myteal!85!white} 0.85 & \cellcolor{myorange!82!white} 0.82 & \cellcolor{myteal!85!white} 0.85 & \cellcolor{myorange!99!white} 0.99 \\
\midrule
 & 0.01 & \cellcolor{myteal!96!white} 0.97 & \cellcolor{myorange!94!white} 0.95 & \cellcolor{myteal!96!white} 0.97 & \cellcolor{myorange!98!white} 0.99 & \cellcolor{myteal!96!white} 0.97 & \cellcolor{myorange!99!white} 0.99 & & \cellcolor{myteal!93!white} 0.94 & \cellcolor{myorange!7!white} 0.07 & \cellcolor{myteal!93!white} 0.94 & \cellcolor{myorange!10!white} 0.10 & \cellcolor{myteal!93!white} 0.94 & \cellcolor{myorange!25!white} 0.25 \\
\multirow{3}{*}{\textbf{Fashion}} & 0.02 & \cellcolor{myteal!96!white} 0.97 & \cellcolor{myorange!97!white} 0.98 & \cellcolor{myteal!96!white} 0.97 & \cellcolor{myorange!99!white} 1.00 & \cellcolor{myteal!96!white} 0.97 & \cellcolor{myorange!99!white} 1.00 & & \cellcolor{myteal!93!white} 0.94 & \cellcolor{myorange!7!white} 0.08 & \cellcolor{myteal!93!white} 0.94 & \cellcolor{myorange!18!white} 0.18 & \cellcolor{myteal!93!white} 0.94 & \cellcolor{myorange!64!white} 0.64 \\
 & 0.05 & \cellcolor{myteal!96!white} 0.97 & \cellcolor{myorange!99!white} 0.99 & \cellcolor{myteal!96!white} 0.97 & \cellcolor{myorange!99!white} 1.00 & \cellcolor{myteal!96!white} 0.97 & \cellcolor{myorange!99!white} 1.00 & & \cellcolor{myteal!93!white} 0.94 & \cellcolor{myorange!9!white} 0.10 & \cellcolor{myteal!93!white} 0.94 & \cellcolor{myorange!61!white} 0.61 & \cellcolor{myteal!93!white} 0.94 & \cellcolor{myorange!97!white} 0.97 \\
\midrule
 & 0.01 & \cellcolor{myteal!71!white} 0.71 & \cellcolor{myorange!52!white} 0.53 & \cellcolor{myteal!71!white} 0.72 & \cellcolor{myorange!96!white} 0.97 & \cellcolor{myteal!71!white} 0.71 & \cellcolor{myorange!99!white} 0.99 & & \cellcolor{myteal!66!white} 0.66 & \cellcolor{myorange!33!white} 0.33 & \cellcolor{myteal!66!white} 0.66 & \cellcolor{myorange!27!white} 0.28 & \cellcolor{myteal!66!white} 0.66 & \cellcolor{myorange!41!white} 0.42 \\
\multirow{3}{*}{\textbf{CIFAR10}} & 0.02 & \cellcolor{myteal!71!white} 0.71 & \cellcolor{myorange!73!white} 0.73 & \cellcolor{myteal!71!white} 0.71 & \cellcolor{myorange!98!white} 0.99 & \cellcolor{myteal!71!white} 0.71 & \cellcolor{myorange!99!white} 1.00 & & \cellcolor{myteal!66!white} 0.66 & \cellcolor{myorange!37!white} 0.38 & \cellcolor{myteal!66!white} 0.66 & \cellcolor{myorange!52!white} 0.52 & \cellcolor{myteal!65!white} 0.66 & \cellcolor{myorange!86!white} 0.87 \\
 & 0.05 & \cellcolor{myteal!70!white} 0.71 & \cellcolor{myorange!92!white} 0.93 & \cellcolor{myteal!70!white} 0.71 & \cellcolor{myorange!99!white} 1.00 & \cellcolor{myteal!70!white} 0.71 & \cellcolor{myorange!99!white} 1.00 & & \cellcolor{myteal!65!white} 0.65 & \cellcolor{myorange!50!white} 0.50 & \cellcolor{myteal!65!white} 0.65 & \cellcolor{myorange!92!white} 0.93 & \cellcolor{myteal!65!white} 0.65 & \cellcolor{myorange!99!white} 1.00 \\
\bottomrule
\end{tabular}%
}
\end{table}

\end{document}